\documentclass{article} 
\usepackage{iclr2027_conference,times}

\usepackage{amsmath,amsfonts,bm}

\def\eqref#1{equation~\ref{#1}}

\def\1{\bm{1}}

\DeclareMathAlphabet{\mathsfit}{\encodingdefault}{\sfdefault}{m}{sl}
\SetMathAlphabet{\mathsfit}{bold}{\encodingdefault}{\sfdefault}{bx}{n}

\usepackage{hyperref}       
\usepackage{url}            
\usepackage{booktabs}       
\usepackage{amsfonts}       
\usepackage{nicefrac}       
\usepackage{microtype}      
\usepackage{xcolor}         
\usepackage{graphicx}
\usepackage{multirow}
\usepackage{amsmath}
\usepackage{natbib}
\usepackage{wrapfig}
\usepackage{enumitem}

\usepackage{amsthm}
\usepackage{amsmath}
\usepackage{amssymb}
 
\newtheorem{theorem}{Theorem}[section]
\newtheorem{proposition}[theorem]{Proposition}
\newtheorem{corollary}[theorem]{Corollary}

\theoremstyle{definition}

\theoremstyle{remark}
\newtheorem{remark}[theorem]{Remark}

\title{\textsc{LongMoE}: Longitudinal Multimodal Learning via Trajectory-Aware Mixture-of-Experts}

\author{Maxx Richard Rahman \\
Saarland University and \\
German Research Center for Artificial Intelligence (DFKI) \\
Saarbr\"{u}cken, Germany \\
\And
Prakhar Kumar \\
Saarland University \\
Saarbr\"{u}cken, Germany \\
\AND
Wolfgang Maass \\
Saarland University and \\
German Research Center for Artificial Intelligence (DFKI) \\
Saarbr\"{u}cken, Germany
}

\iclrfinalcopy 
\begin{document}

\maketitle
\lhead{Preprint}

\begin{abstract}
Multimodal clinical learning is increasingly important for integrating diverse patient data, including imaging, text, and personalised health records. However, it faces two fundamental challenges: i) modality missingness, where arbitrary subsets of modalities are unavailable at a given patient visit, ii) longitudinal dynamics, where the diagnostic significance of an observation depends on the patient's evolving disease trajectory over time. Existing methods address these challenges in isolation: missing-modality frameworks treat each visit as an independent static snapshot and discard temporal context, while longitudinal models often assume complete modality availability and degrade under systematic modality incompleteness. We propose \textsc{LongMoE} (\underline{Long}itudinal \underline{M}ixture-\underline{o}f-\underline{E}xperts), the unified framework to jointly address both challenges. \textsc{LongMoE} combines a context-aware imputation module with an attentional tokenization module that captures frequency-domain temporal patterns across irregular visit sequences, a trajectory-aware encoder for modeling disease progression, and context-conditioned Sparse MoE routing for patient-specific expert selection. Experiments on ADNI, OASIS-3, and MIMIC-IV show that \textsc{LongMoE} improves robustness under missing or weak contemporaneous modalities and remains competitive in full-modality settings, establishing a strong foundation for longitudinally-aware multimodal clinical learning.
\end{abstract}

\section{Introduction}
\label{sec:introduction}

Multimodal learning has become a fundamental challenge in clinical decision support, where integrating complementary data sources yields richer and more reliable predictions than any single modality alone~\citep{baltrusaitis2018multimodal, liang2021multibench}. In Alzheimer's disease (AD), this integrative capacity is clinically necessary, i.e., pathology manifests concurrently across structural neuroimaging, functional biomarkers, cognitive assessments, and genetic risk profiles, each contributing distinct and partially non-overlapping diagnostic signal~\citep{jack2013biomarker, marquez2019neuroimaging, papassotiropoulos2006genetics}. In practice, real-world cohorts exhibit pervasive \emph{modality missingness}: any given patient visit yields only a subset of the full modality panel due to heterogeneous acquisition protocols, resource constraints, and time-dependent assay availability~\citep{yun2024flexmoe, han2024fusemoe}. Existing approaches either restrict training to fully observed patients or apply context-free imputation strategies such as zero-padding and global mean substitution~\citep{zhang2022mmformer, wang2024sharedspecific}. A compounding difficulty is that AD patients are not a homogeneous population, i.e., individuals differ systematically in genetic profiles, comorbidities, and progression rates, so that the relative informativeness of each modality varies substantially across subgroups~\citep{jack2013biomarker}. With $L$ modalities, a model should handle up to $2^L - 1$ distinct observation patterns, each inducing a different information structure over which fusion must be performed. A monolithic fusion network is structurally unsuited to this setting~\citep{fedus2022switch}, motivating a Sparse Mixture-of-Experts (SMoE) architecture~\citep{shazeer2017outrageously, jacobs1991adaptive} in which a learned router dynamically activates a sparse subset of specialised expert networks conditioned on the observed modality combination, providing modality-pattern-aware and subgroup-aware fusion without the cost of a fully dense model.

Alzheimer's disease is a slowly progressive condition unfolding over years to decades, and the diagnostic significance of any observation is inseparable from its longitudinal context~\citep{weiner2010adni, venugopalan2021multimodal}. What matters clinically is not the absolute value of a biomarker at a single visit but its position along the patient's individual trajectory of decline. For example, a moderate hippocampal volume reading carries an entirely different prognostic weight depending on whether it represents a stable plateau or a steep atrophy curve~\citep{jack2013biomarker}. Therefore, \emph{the rate and pattern of biomarker change} across successive visits is a substantially stronger predictor of disease stage than any cross-sectional measurement~\citep{weiner2010adni, lee2019predicting}. GRU-based models~\citep{cho2014gru, lee2019predicting} capture longitudinal dependencies but require complete modality inputs at every visit and fail under structured missingness, while Transformer-based models~\citep{vaswani2017attention, devlin2019bert} offer greater expressivity but have not been adapted to the joint challenge of irregular visit spacing, arbitrary modality combinations, and population-level heterogeneity.
 
Recent work Flex-MoE~\citep{yun2024flexmoe} represents the closest prior work, introducing an SMoE framework with a learnable missing-modality bank and a two-stage routing strategy that accommodates arbitrary modality subsets. Nevertheless, Flex-MoE and all concurrent missing-modality methods~\citep{han2024fusemoe, tsai2019multimodal, zadeh2017tensor} treat each clinical visit as an independent static snapshot, that means the routing mechanism receives no information about the patient's prior visits, biomarker history, or evolving modality availability, and is therefore structurally incapable of trajectory-aware fusion. The result is a precise and consequential gap: \emph{no existing model jointly addresses modality missingness and longitudinal dynamics}. We address this gap with \textsc{LongMoE}, whose contributions are as follows:

\begin{itemize}[itemsep=2pt,topsep=3pt,leftmargin=*]
    \item We propose \textsc{LongMoE}, a unified framework that jointly addresses modality missingness and longitudinal temporal dynamics in clinical multimodal learning. 

    \item We introduce an attentional tokenization module that integrates modality-level cross-attention with continuous-time, multi-frequency positional encoding. This encoding provides a temporal representation for non-uniformly sampled visit sequences.

    \item Through extensive experiments on ADNI, OASIS-3, and MIMIC-IV, we show that \textsc{LongMoE} consistently outperforms state-of-the-art baselines under missing contemporaneous modalities.
\end{itemize}


\section{Related Work}
\label{sec:related_work}

\paragraph{Missing Modality Handling in Multimodal Learning.}
\label{sec:related_multimodal}

Multimodal learning has been extensively studied across vision, language, and healthcare domains~\citep{baltrusaitis2018multimodal, liang2021multibench, rahman2022data, john2024llmrs}, with AD-specific work confirming that combining imaging, genetic, and clinical data consistently outperforms any single modality~\citep{venugopalan2021multimodal, odusami2023machine}. Canonical fusion architectures including TF~\citep{zadeh2017tensor} and MulT~\citep{tsai2019multimodal} provide strong baselines but uniformly assume complete modality availability, while global imputation strategies such as zero-padding~\citep{little2019statistical} and learned methods including ShaSpec~\citep{wang2024sharedspecific} and mmFormer~\citep{zhang2022mmformer} condition imputation on population statistics rather than individual patient context, limiting effectiveness for rare modality combinations. The most directly relevant prior works are Flex-MoE~\citep{yun2024flexmoe} and FuseMoE~\citep{han2024fusemoe}. Flex-MoE handles arbitrary modality combinations via a learnable missing modality bank and a two-stage SMoE design (G-Router for full-modality; S-Router for incomplete samples). FuseMoE applies dynamic expert input masking for variable modality sets, but both treat each visit as an independent static snapshot, discarding the patient's historical trajectory entirely.

\paragraph{Longitudinal Modelling and Mixture-of-Experts.}
\label{sec:related_longitudinal_moe}

Longitudinal modelling of disease progression has relied on recurrent architectures~\citep{hochreiter1997lstm, cho2014gru, choi2016retain}, including the GRU-based multimodal model Lee-MMGRU~\citep{lee2019predicting}, which captures temporal dynamics across imaging, clinical, and genetic modalities in AD but requires complete modality observations at every visit, rendering it unable to exploit partially observed patients. Transformer-based sequence models~\citep{vaswani2017attention, devlin2019bert, rahman2024sacnn} have largely superseded recurrent architectures, with irregular time-series Transformers~\citep{horn2020sett, shukla2021multitime, rahman2023modelling} demonstrating that attention accommodates irregular sampling via learned positional encodings. The Mixture-of-Experts paradigm~\citep{jacobs1991adaptive, jordan1994hierarchical} has undergone a resurgence through Sparsely-Gated MoE~\citep{shazeer2017outrageously}, Switch Transformers~\citep{fedus2022switch}, and Mixtral~\citep{jiang2024mixtral}. In the multimodal setting, LiMoE~\citep{mustafa2022limoe}, SM$^4$~\citep{peng2024sparse}, MMoE~\citep{ma2018modeling}, Mod-Squad~\citep{chen2023modsquad}, DSelect-$k$~\citep{hazimeh2021dselectk}, Expert Choice~\citep{zhou2022expertchoice}, and sparse vision MoE~\citep{riquelme2021scaling} advanced expert specialisation and load-balanced routing, but none addresses the intersection of longitudinal sequence modelling and missing modality handling on irregular, partially observed clinical visit sequences.

\section{Preliminaries}
\label{sec:preliminaries}

\paragraph{Longitudinal Patient Record.}
Let a patient's clinical history be represented as an ordered sequence of $T$ visits, $\mathcal{V} = \{v_1, v_2, \ldots, v_T\}$, where each visit $v_t$ occurs at timestamp $\tau_t \in \mathbb{R}_{>0}$.  The inter-visit intervals $\Delta_t = \tau_t - \tau_{t-1}$ are irregular and patient-specific, reflecting the non-uniform nature of real-world clinical data collection.

\paragraph{Multimodal Observations.}
At each visit $v_t$ a subset of $L$ modalities is observed.  Let $\mathcal{M} = \{m_1, m_2, \ldots, m_L\}$ denote the full set of available modalities.  For the Alzheimer's disease progression modeling, these comprise: \emph{(i)} structural MRI scans $\mathbf{x}^{\mathrm{mri}}_t \in \mathbb{R}^{512\times 512}$, \emph{(ii)} clinical assessment scores $\mathbf{x}^{\mathrm{clin}}_t \in \mathbb{R}^{64}$, \emph{(iii)} CSF/biospecimen panels $\mathbf{x}^{\mathrm{bio}}_t \in \mathbb{R}^{b}$, and \emph{(iv)} genetic assay data $\mathbf{x}^{\mathrm{gen}}_t \in \mathbb{R}^{g}$. In practice, not all modalities are observed at every visit. Any given visit may yield only one or two modalities, making static full-input assumptions unrealistic.  Let $\mathbf{O}_t \in \{0,1\}^L$ be a binary observation mask at visit $t$, where $O_{t,l}=1$ if modality $m_l$ is observed and $O_{t,l}=0$ otherwise. Given the full longitudinal record of a patient up to and including the current visit $v_t$, the goal is to predict the diagnostic label $y_t \in \mathcal{Y} = \{\text{CN},\, \text{MCI},\, \text{AD}\}$, where CN denotes cognitively normal, MCI denotes mild cognitive impairment, and AD denotes Alzheimer's disease.  

\section{LongMoE Architecture}
\label{sec:architecture}

\begin{figure*}[t]
  \centering
  \includegraphics[width=\textwidth]{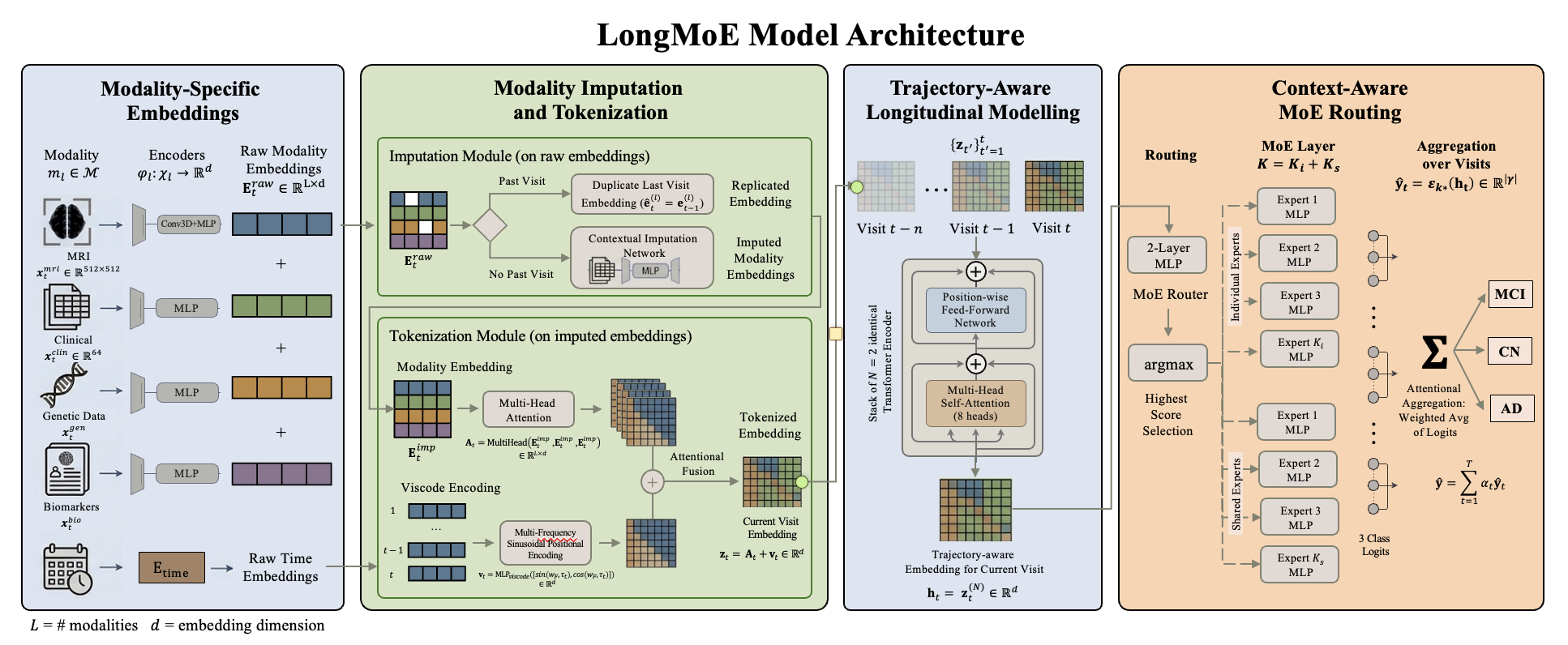}
  \caption{Overview of the \textsc{LongMoE} architecture, which processes a patient's longitudinal visit sequence through four sequential modules:\emph{(i)} modality-specific embedding, \emph{(ii)} a modality imputation and tokenization, \emph{(iii)} a trajectory-aware longitudinal modeling, and \emph{(iv)} a context-aware MoE routing.}
  \label{fig:architecture}
\end{figure*}

\textsc{LongMoE} processes longitudinal multimodal patient data through four sequentially composed modules (Figure~\ref{fig:architecture}): \emph{(i)} modality-specific embedding, \emph{(ii)} a modality imputation and tokenization, \emph{(iii)} a trajectory-aware longitudinal modeling, and \emph{(iv)} a context-aware MoE routing.

\subsection{Modality-Specific Embeddings}
\label{sec:embeddings}

For each modality $m_l\in\mathcal{M}$ we define an encoder $\phi_l:\mathcal{X}_l\to\mathbb{R}^d$, where $\mathcal{X}_l$ is the native input space of modality $l$.  Encoder architectures are chosen to match the structural properties of each modality: 

\begin{itemize}[itemsep=2pt,topsep=3pt,leftmargin=*]
  \item \textit{MRI}
        ($\mathbf{x}^{\mathrm{mri}}_t\in\mathbb{R}^{512\times 512}$):
        A 3D convolutional encoder followed by an MLP projection head
        captures local volumetric structure and global morphological
        features $\mathbf{e}^{\mathrm{mri}}_t = \mathrm{MLP}\!\left(\mathrm{Conv3D}\!\left( \mathbf{x}^{\mathrm{mri}}_t\right)\right) \in \mathbb{R}^d.$

  \item \textit{Clinical Scores}
        ($\mathbf{x}^{\mathrm{clin}}_t\in\mathbb{R}^{64}$),
        \textit{Biomarkers} ($\mathbf{x}^{\mathrm{bio}}_t$), and
        \textit{Genetic Data} ($\mathbf{x}^{\mathrm{gen}}_t$):
        Each is processed by a modality-specific MLP $\mathbf{e}^{(l)}_t = \mathrm{MLP}_l\!\left(\mathbf{x}^{(l)}_t\right) \in \mathbb{R}^d, \quad l \in \{\mathrm{clin},\,\mathrm{bio},\,\mathrm{gen}\}.$

  \item \textit{Temporal Embedding}: The visit timestamp $\tau_t$ is encoded
        as a learned temporal embedding
        $\mathbf{E}_{\mathrm{time}}\in\mathbb{R}^d$, capturing both
        absolute time and the inter-visit interval $\Delta_t$.
\end{itemize}

The per-modality embeddings at visit $t$ are stacked to form $\mathbf{E}^{\mathrm{raw}}_t = \bigl[\mathbf{e}^{(1)}_t;\;\mathbf{e}^{(2)}_t;\;\ldots;\; \mathbf{e}^{(L)}_t\bigr] \in \mathbb{R}^{L\times d}$, where rows corresponding to unobserved modalities ($O_{t,l}=0$) are set to zero prior to imputation.  All encoders are jointly trained end-to-end with the full model.

\subsection{Modality Imputation and Tokenization}
\label{sec:imputation_tokenization}

\paragraph{Imputation Module.}
\label{sec:imputation}

For each missing modality $l$ at visit $t$ (i.e., $O_{t,l}=0$) we apply
one of two imputation strategies conditioned on the availability of past
visit information.

\textit{Past Visit Available ($t>1$ and $O_{t-1,l}=1$).} We carry forward the most recent observed embedding via a replication operation $\hat{\mathbf{e}}^{(l)}_t = \mathbf{e}^{(l)}_{t-1} \qquad \text{(Duplicate Last Visit Embedding).}$ This strategy is motivated by the clinical observation that biomarker and imaging values exhibit strong short-term autocorrelation, making the last observed value a more informative prior than the global mean.

\textit{No Past Visit Available ($t=1$ or no prior observation exists for modality $l$).} We apply a learned contextual imputation network that conditions on the available modalities at the current visit $\hat{\mathbf{e}}^{(l)}_t = \mathrm{MLP}_{\mathrm{imp}}\!\left( \bigoplus_{l':\,O_{t,l'}=1} \mathbf{e}^{(l')}_t \right),$ where $\bigoplus$ denotes mean-pooling over observed modality embeddings and $\mathrm{MLP}_{\mathrm{imp}}$ is a shared contextual imputation network with learnable zero-masking, preventing contamination of observed modalities by missing signal. The fully imputed embedding matrix is defined as:
\begin{equation}
  \mathbf{E}^{\mathrm{imp}}_t[l,\,:]
  = \begin{cases}
      \mathbf{e}^{(l)}_t              & \text{if } O_{t,l}=1, \\
      \hat{\mathbf{e}}^{(l)}_t        & \text{if } O_{t,l}=0.
    \end{cases}
  \label{eq:imputed_matrix}
\end{equation}

\paragraph{Tokenization Module.}
\label{sec:tokenization}

Given $\mathbf{E}^{\mathrm{imp}}_t$, we produce a unified visit-level token $\mathbf{z}_t\in\mathbb{R}^d$ via attentional fusion of two complementary representations: a \emph{modality embedding} and a \emph{viscode encoding}.

\textit{Modality Embedding via Multi-Head Attention.} We apply multi-head self-attention over the $L$ modality embeddings $\mathbf{A}_t = \mathrm{MultiHead}\!\left( \mathbf{E}^{\mathrm{imp}}_t,\; \mathbf{E}^{\mathrm{imp}}_t,\; \mathbf{E}^{\mathrm{imp}}_t \right) \in \mathbb{R}^{L\times d},$ where $\mathrm{MultiHead}(\mathbf{Q},\mathbf{K},\mathbf{V}) = \mathrm{Concat}(\mathrm{head}_1,\ldots,\mathrm{head}_H)\mathbf{W}^O$ with each head defined as: 
\begin{equation}
    \mathrm{head}_h = \mathrm{Softmax}\!\left( \frac{\mathbf{Q}\mathbf{W}^Q_h\, (\mathbf{K}\mathbf{W}^K_h)^\top} {\sqrt{d/H}} \right)\mathbf{V}\mathbf{W}^V_h.
    \label{eqn:cross_attention}
\end{equation}
This operation learns inter-modality dependencies, enabling, e.g., MRI-derived morphological features to modulate the interpretation of clinical assessment scores.

\textit{Viscode Encoding via Multi-Frequency Sinusoidal Positional Encoding.} 
To encode the temporal position of each visit in the patient's longitudinal sequence, we perform a continuous-time multi-frequency sinusoidal encoding that maps each visit's absolute timestamp $\tau_t \in \mathbb{R}_{>0}$ (e.g., months from enrolment) to a fixed-dimensional vector without assuming uniform visit spacing. For $F$ frequency bands with base frequencies $\omega_k = 2^{k-1}$ for $k = 1, \ldots, F$, the viscode encoding of visit $t$ is:
\begin{equation}
  \mathbf{v}_t
  = \mathrm{MLP}_{\mathrm{viscode}}\!\left(
    \left[\,
      \sin(\omega_1\,\tau_t),\;
      \cos(\omega_1\,\tau_t),\;
      \ldots,\;
      \sin(\omega_F\,\tau_t),\;
      \cos(\omega_F\,\tau_t)
    \,\right]
  \right) \in \mathbb{R}^d,
  \label{eq:viscode_sincos}
\end{equation}
yielding a $2F$-dimensional raw encoding that is projected to $\mathbb{R}^d$ by the downstream MLP. To see why multiple frequencies are necessary, note that a single sine function is not injective: $\sin(\omega\tau) = 0$ at both $\tau = 0$ and $\tau = \pi/\omega$, so two visits at those months would receive identical encodings. By concatenating $F$ sine-cosine pairs at geometrically spaced frequencies, two visits $\tau$ and $\tau'$ receive the same encoding only if $\sin(\omega_k\tau) = \sin(\omega_k\tau')$ and $\cos(\omega_k\tau) = \cos(\omega_k\tau')$ hold simultaneously for all $k = 1, \ldots, F$, which requires $\omega_k(\tau - \tau') \equiv 0 \pmod{2\pi}$ for all $k$. Since $\tau, \tau' \in \mathbb{Z}_{>0}$ (integer month indices), $2\pi$ is irrational, and the frequencies $\omega_k$ are chosen to be incommensurable, this collision cannot occur within any clinically plausible range of visit months.The downstream $\mathrm{MLP}_{\mathrm{viscode}}$ then learns a task-specific projection of these multi-frequency components, enabling the model to weight slow trends (low $\omega_k$) versus rapid rate changes (high $\omega_k$) adaptively for the downstream classification task.

\textit{Attentional Fusion.}
The modality embedding $\bar{\mathbf{A}}_t\in\mathbb{R}^d$ and viscode encoding $\mathbf{v}_t$ are fused via elementwise addition $\mathbf{z}_t = \bar{\mathbf{A}}_t + \mathbf{v}_t \in \mathbb{R}^d.$ This additive fusion preserves gradient flow from both branches and has been shown to be effective in multimodal fusion settings~\cite{kiela2019supervised}.

\subsection{Trajectory-Aware Longitudinal Modelling}
\label{sec:longitudinal}

Given the sequence of visit tokens $\{\mathbf{z}_{t'}\}_{t'=1}^{t}$, we produce a contextualised representation $\mathbf{h}_t\in\mathbb{R}^d$ that integrates information across all past visits, capturing the patient's disease-progression trajectory. We apply a stack of $N=2$ identical Transformer encoder blocks to model temporal dependencies across visits.  Each block consists of a multi-head self-attention sublayer (with $H=8$ heads) followed by a position-wise feed-forward network (FFN), with residual connections and layer normalisation~\citep{vaswani2017attention}. To prevent information leakage from future visits, we apply a causal attention mask: position $t'$ may attend only to positions $t''\le t'$. This ensures that the prediction at visit $t$ is conditioned solely on the history $\{v_1,\ldots,v_t\}$. The output of the final encoder layer at position $t$ yields $\mathbf{h}_t = \mathbf{z}^{(N)}_t \in \mathbb{R}^d.$ This embedding encodes not merely the current visit's observations but a contextualised summary of the patient's full longitudinal trajectory, providing a rich input to the downstream MoE routing module.

\subsection{Context-Aware MoE Routing}
\label{sec:moe}

\paragraph{Expert Architecture.}
The MoE layer comprises $K = K_i + K_s$ experts partitioned into two groups: i) \textit{Individual Experts} $\{\mathcal{E}^{\mathrm{ind}}_k\}_{k=1}^{K_i}$: Specialised sub-networks that capture distinct patient-subpopulation patterns. Each expert is an MLP $\mathcal{E}^{\mathrm{ind}}_k(\mathbf{h}_t) = \mathrm{MLP}_k(\mathbf{h}_t) \in \mathbb{R}^{|\mathcal{Y}|}.$ ii) \textit{Shared Experts} $\{\mathcal{E}^{\mathrm{shr}}_k\}_{k=1}^{K_s}$: Experts accessible to \emph{all} inputs that capture population-level common patterns, providing a regularisation effect and preventing over-specialisation. In total the MoE layer contains $K=8$ experts, each producing a 3-class
logit vector.

\paragraph{MoE Router.}
The router is a two-layer MLP mapping $\mathbf{h}_t$ to a scoring vector over all $K$ experts $\mathbf{s}_t = \mathrm{Router}(\mathbf{h}_t) = \mathrm{MLP}_{\mathrm{router}}(\mathbf{h}_t) \in \mathbb{R}^K.$ Expert selection is performed via hard top-1 routing $k^* = \arg\max_{k\in\{1,\ldots,K\}} s_{t,k}.$ The shared experts $\{\mathcal{E}^{\mathrm{shr}}_k\}$ are always activated for all inputs, providing a stable common representation. Hard top-1 routing imposes sparsity in expert utilisation, reducing computational cost while encouraging expert specialisation, in contrast to soft routing, which can lead to diffuse utilisation and slower specialisation.  The shared expert mechanism ensures that even under hard routing, all patients benefit from a common representational backbone, mitigating the cold-start problem for rare patient subgroups.

\paragraph{Aggregation over Visits.}
Let $\hat{\mathbf{y}}_t = \mathcal{E}_{k^*}(\mathbf{h}_t)\in\mathbb{R}^{|\mathcal{Y}|}$ denote the logit vector produced by the selected expert at visit $t$.  The final prediction is obtained via attentional aggregation:
\begin{equation}
  \hat{\mathbf{y}} = \sum_{t=1}^{T} \alpha_t\,\hat{\mathbf{y}}_t,
  \qquad
  \alpha_t = \frac{\exp\!\left(s_{t,k^*_t}\right)}
                  {\sum_{t'=1}^{T}\exp\!\left(s_{t',k^*_{t'}}\right)},
  \label{eq:aggregation}
\end{equation}
where $\alpha_t$ weights each visit's contribution by the router's confidence, assigning greater influence to visits with more informative or reliable observations. The aggregated logit vector $\hat{\mathbf{y}}$ is passed through a softmax $\hat{p}(y\mid\mathcal{V}) = \mathrm{Softmax}(\hat{\mathbf{y}}) \in \Delta^{|\mathcal{Y}|-1},$ and the model is trained with the standard cross-entropy loss $\mathcal{L} = -\sum_{c\in\mathcal{Y}} \mathbb{1}[y=c]\,\log\hat{p}(c\mid\mathcal{V}).$

\paragraph{Load-Balancing Regularisation.}
To prevent router collapse, we augment the training objective with an auxiliary load-balancing loss~\citep{fedus2022switch} $\mathcal{L}_{\mathrm{bal}} = K \cdot \sum_{k=1}^{K} f_k \cdot P_k,$ where $f_k = \frac{1}{|\mathcal{B}|}\sum_{i\in\mathcal{B}}\mathbb{1}[k^*_i=k]$ is the fraction of inputs in mini-batch $\mathcal{B}$ routed to expert $k$, and $P_k = \frac{1}{|\mathcal{B}|}\sum_{i\in\mathcal{B}} s_{i,k}$ is the mean routing score for expert $k$.  The total training objective is $\mathcal{L}_{\mathrm{total}} = \mathcal{L} + \lambda\,\mathcal{L}_{\mathrm{bal}},$ where $\lambda>0$ is a regularisation coefficient controlling the trade-off between task performance and expert load balance.

\section{Experiments}
\label{sec:experiments}

\paragraph{ADNI.}~\citep{weiner2010adni}
The Alzheimer's Disease Neuroimaging Initiative comprises of 2,693 patients across CN (1,049; 38.9\%), MCI (1,194; 44.3\%), and
AD (450; 16.7\%), with four modalities: MRI/PET~\citep{doshi2016muse} registered to MNI
space~\citep{ou2011dramms} (2,359; 87.6\%); CSF biomarkers
(A$\beta$40, A$\beta$42, total/phospho-Tau, ApoE genotyping),
normalised to $[-1,1]$ with categoricals one-hot encoded (1,611;
59.8\%); longitudinal clinical assessments from five source files
with label-leaking columns excluded (2,684; 99.7\%); and SNP data
from phases 1/GO/2/3 harmonised to GRCh38 via LiftOver and LD-pruned
(window 50, step 5, $r^2 < 0.1$) (2,226; 82.7\%).
 
\paragraph{OASIS-3.}~\citep{lamontagne2019oasis}
This open-access longitudinal dataset from Washington University comprises of 1,377 participants (aged 42-95) across CN (857; 62.2\%), MCI (203;
14.7\%), and AD (317; 23.0\%), with T1-weighted MRI (1,098; 79.7\%), CSF amyloid/tau assays and PET amyloid burden scores
(Pittsburgh Compound-B and Florbetapir), CDR/MMSE assessments (1,098;
79.7\%), and ApoE $\varepsilon$4 with select AD-associated SNPs (620;
45.0\%). OASIS-3 presents complementary challenges through higher
visit spacing irregularity, more severe modality missingness, and a markedly
more skewed class distribution; prediction target:
$y \in \{\mathrm{CN}, \mathrm{MCI}, \mathrm{AD}\}$.
 
\paragraph{MIMIC-IV.}~\citep{johnson2020mimic}
98,323 adults with at least two hospital visits at Beth Israel
Deaconess Medical Center, with four modalities: clinical notes (N)
encoded via a pre-trained clinical language model (98,323; 100\%);
lab values and vitals aggregated over the last visit window (80,802;
82.2\%); ICD-9 codes as multi-hot vectors (96,689; 98.3\%); and
structured demographics, i.e., age, sex, ethnicity, insurance type
(98,323; 100\%). Binary one-year mortality under severe class
imbalance (Survived: 79,266, 80.6\%; Died: 19,057, 19.4\%) serves
as a large-scale stress test for generalisation to a fundamentally
different clinical task with distinct modality semantics.

\paragraph{Baselines.}
We compare \textsc{LongMoE} against five baselines under identical splits and missing-modality protocols. Flex-MoE~\citep{yun2024flexmoe} handles arbitrary modality combinations via a learnable missing-modality bank and a two-stage SMoE (G-Router/S-Router), augmented with a mean-pooled preceding-visit summary for longitudinal fairness. FuseMoE~\citep{han2024fusemoe} is a concurrent MoE Transformer augmented with a single-layer GRU over per-visit embeddings. TF~\citep{zadeh2017tensor} computes outer-product tensor interactions, zero-pads absent modalities, and has no temporal capacity. MulT~\citep{tsai2019multimodal} uses directional cross-modal attention but is designed for language-vision sequences and collapses under systematic missingness. Lee-MMGRU~\citep{lee2019predicting} is the only baseline with explicit longitudinal modelling via a GRU encoder, but requires all modalities at every visit and thus cannot exploit partially observed patients.

\paragraph{Experimental Settings.}
All models are trained on NVIDIA H100 GPUs (80\,GB) with a patient-level stratified 70/15/15 train/validation/test split, keeping all visits of a patient in one partition. \textsc{LongMoE} uses $d = 128$, a 2-layer, 8-head Transformer encoder, $K = 8$ MoE experts, AdamW~\citep{loshchilov2019decoupled}, and $F = 8$ frequency bands (16-dimensional raw viscode encoding). To prevent temporal leakage, each test patient is evaluated on their final visit, with all prior visits as context for the encoder and imputation module, mirroring clinical deployment. AUC-ROC is the primary metric; accuracy and F1 are reported in Appendix~\ref{app:full_results}. Under current-visit masking, an excluded modality is withheld only at the evaluated visit while prior visits retain their observed modalities, so each combination describes the prediction visit rather than the full record. Genetics, a static per-patient attribute, is the only modality removed from the entire record when excluded. Visit-independent baselines lack history, so their single-modality performance is bounded by that modality's standalone signal. We quantify this intrinsic asymmetry in Appendix~\ref{app:fullTraj} under full-trajectory masking.

 
\begin{table}[ht]
  \centering
  \caption{%
    Performance comparison on ADNI (AUC-ROC)
    for all modality combinations.
    $^\dagger$ denotes statistical significance over the runner-up
    baseline (paired $t$-test, Bonferroni-corrected, $p < 0.05$).
  }
\label{tab:adni_perf}
  \renewcommand{\arraystretch}{1.15}
  \resizebox{0.9\textwidth}{!}{%
  \begin{tabular}{lcccccc}
    \toprule
    \textbf{Modalities}
      & \textbf{MulT}
      & \textbf{Flex-MoE}
      & \textbf{FuseMoE}
      & \textbf{Lee-MMGRU}
      & \textbf{TF}
      & \textbf{LongMoE} \\
    \midrule
    B
      & $0.5586_{\pm.0170}$ & $0.5220_{\pm.1096}$ & $0.4915_{\pm.0230}$
      & $0.5000_{\pm.0000}$ & $0.6419_{\pm.0127}$
      & $\mathbf{0.9493_{\pm.0040}}^\dagger$ \\
    C
      & $0.9619_{\pm.0060}$ & $0.9650_{\pm.0016}$ & $0.9647_{\pm.0012}$
      & $0.9704_{\pm.0034}$ & $0.9250_{\pm.0192}$
      & $\mathbf{0.9750_{\pm.0027}}^\dagger$ \\
    G
      & $0.6067_{\pm.0193}$ & $0.6011_{\pm.0305}$ & $0.6316_{\pm.0059}$
      & $0.5000_{\pm.0000}$ & $0.6393_{\pm.0126}$
      & $\mathbf{0.9517_{\pm.0036}}^\dagger$ \\
    I
      & $0.5884_{\pm.0166}$ & $0.6632_{\pm.0199}$ & $0.6719_{\pm.0134}$
      & $0.5771_{\pm.0696}$ & $0.7137_{\pm.0109}$
      & $\mathbf{0.9534_{\pm.0032}}^\dagger$ \\
    \midrule
    B,C
      & $0.9638_{\pm.0047}$ & $0.9573_{\pm.0048}$ & $0.9632_{\pm.0010}$
      & $0.9704_{\pm.0034}$ & $0.9262_{\pm.0121}$
      & $\mathbf{0.9749_{\pm.0026}}^\dagger$ \\
    B,G
      & $0.6286_{\pm.0088}$ & $0.6293_{\pm.0426}$ & $0.6244_{\pm.0056}$
      & $0.5000_{\pm.0000}$ & $0.6911_{\pm.0155}$
      & $\mathbf{0.9515_{\pm.0029}}^\dagger$ \\
    C,G
      & $0.9712_{\pm.0023}$ & $0.9490_{\pm.0043}$ & $0.9679_{\pm.0011}$
      & $0.9704_{\pm.0034}$ & $0.9397_{\pm.0097}$
      & $\mathbf{0.9755_{\pm.0030}}^\dagger$ \\
    I,B
      & $0.5985_{\pm.0154}$ & $0.6645_{\pm.0233}$ & $0.6707_{\pm.0108}$
      & $0.5771_{\pm.0696}$ & $0.7502_{\pm.0157}$
      & $\mathbf{0.9531_{\pm.0029}}^\dagger$ \\
    I,C
      & $0.9584_{\pm.0069}$ & $0.9526_{\pm.0045}$ & $0.9571_{\pm.0040}$
      & $0.9685_{\pm.0037}$ & $0.9194_{\pm.0134}$
      & $\mathbf{0.9740_{\pm.0034}}^\dagger$ \\
    I,G
      & $0.6254_{\pm.0215}$ & $0.7043_{\pm.0271}$ & $0.7151_{\pm.0161}$
      & $0.5771_{\pm.0696}$ & $0.7473_{\pm.0124}$
      & $\mathbf{0.9544_{\pm.0033}}^\dagger$ \\
    \midrule
    B,C,G
      & $0.9712_{\pm.0018}$ & $0.9513_{\pm.0053}$ & $0.9670_{\pm.0012}$
      & $0.9704_{\pm.0034}$ & $0.9398_{\pm.0111}$
      & $\mathbf{0.9754_{\pm.0029}}^\dagger$ \\
    I,B,C
      & $0.9608_{\pm.0058}$ & $0.9517_{\pm.0075}$ & $0.9549_{\pm.0036}$
      & $0.9685_{\pm.0037}$ & $0.9248_{\pm.0133}$
      & $\mathbf{0.9739_{\pm.0033}}^\dagger$ \\
    I,B,G
      & $0.6297_{\pm.0200}$ & $0.7060_{\pm.0245}$ & $0.7123_{\pm.0144}$
      & $0.5771_{\pm.0696}$ & $0.7629_{\pm.0121}$
      & $\mathbf{0.9543_{\pm.0029}}^\dagger$ \\
    I,C,G
      & $0.9665_{\pm.0027}$ & $0.9552_{\pm.0039}$ & $0.9617_{\pm.0038}$
      & $0.9685_{\pm.0037}$ & $0.9356_{\pm.0080}$
      & $\mathbf{0.9744_{\pm.0032}}^\dagger$ \\
    \midrule
    I,B,C,G
      & $0.9666_{\pm.0023}$ & $0.9535_{\pm.0065}$ & $0.9604_{\pm.0037}$
      & $0.9685_{\pm.0037}$ & $0.9320_{\pm.0114}$
      & $\mathbf{0.9742_{\pm.0032}}^\dagger$ \\
    \bottomrule
  \end{tabular}%
  }
\end{table}

 
\begin{table}[ht]
  \centering
  \caption{%
    Performance comparison on OASIS-3 (AUC-ROC)
    for all modality combinations.
    $^\dagger$ denotes statistical significance over the runner-up
    baseline (paired $t$-test, Bonferroni-corrected, $p < 0.05$).
  }
  \label{tab:oasis_perf}
  \renewcommand{\arraystretch}{1.15}
  \resizebox{0.9\textwidth}{!}{%
  \begin{tabular}{lcccccc}
    \toprule
    \textbf{Modalities}
      & \textbf{MulT}
      & \textbf{Flex-MoE}
      & \textbf{FuseMoE}
      & \textbf{Lee-MMGRU}
      & \textbf{TF}
      & \textbf{LongMoE} \\
    \midrule
    C
      & $0.4913_{\pm.0010}$
      & $0.5007_{\pm.0031}$
      & $0.4991_{\pm.0027}$
      & $\mathbf{0.8776_{\pm.0031}}$
      & $0.7364_{\pm.0243}$
      & $0.8630_{\pm.0191}$ \\
    G
      & $0.5040_{\pm.0114}$
      & $0.4993_{\pm.0030}$
      & $0.5016_{\pm.0022}$
      & $0.5264_{\pm.0470}$
      & $0.6533_{\pm.0255}$
      & $\mathbf{0.7624_{\pm.0078}}^\dagger$ \\
    I
      & $0.5352_{\pm.0035}$
      & $0.5313_{\pm.0017}$
      & $0.5309_{\pm.0012}$
      & $0.6760_{\pm.0199}$
      & $0.6846_{\pm.0164}$
      & $\mathbf{0.7818_{\pm.0202}}^\dagger$ \\
    \midrule
    C,G
      & $0.6591_{\pm.0179}$
      & $0.6722_{\pm.0062}$
      & $0.6292_{\pm.0114}$
      & $\mathbf{0.8674_{\pm.0036}}$
      & $0.7498_{\pm.0253}$
      & $0.8643_{\pm.0178}$ \\
    I,C
      & $0.7589_{\pm.0083}$
      & $0.7651_{\pm.0037}$
      & $0.7663_{\pm.0035}$
      & $\mathbf{0.8568_{\pm.0111}}$
      & $0.7408_{\pm.0261}$
      & $0.8521_{\pm.0208}$ \\
    I,G
      & $0.6320_{\pm.0109}$
      & $0.6632_{\pm.0107}$
      & $0.6388_{\pm.0156}$
      & $0.6672_{\pm.0359}$
      & $0.7153_{\pm.0143}$
      & $\mathbf{0.7767_{\pm.0130}}^\dagger$ \\
    \midrule
    I,C,G
      & $0.7713_{\pm.0100}$
      & $0.8030_{\pm.0062}$
      & $0.7811_{\pm.0200}$
      & $0.8330_{\pm.0066}$
      & $0.7562_{\pm.0155}$
      & $\mathbf{0.8552_{\pm.0192}}^\dagger$ \\
    \bottomrule
  \end{tabular}%
  }
\end{table}

\begin{table}[ht]
  \centering
\caption{%
    Performance comparison on MIMIC-IV (AUC-ROC)
    for all modality combinations.
    $^\dagger$ denotes statistical significance over the runner-up
    baseline (paired $t$-test, Bonferroni-corrected, $p < 0.05$).
  }
  \label{tab:mimic_perf}
  \renewcommand{\arraystretch}{1.15}
  \resizebox{0.9\textwidth}{!}{%
  \begin{tabular}{lcccccc}
    \toprule
    \textbf{Modalities} & \textbf{MulT} & \textbf{Flex-MoE}
      & \textbf{FuseMoE} & \textbf{Lee-MMGRU}
      & \textbf{TF} & \textbf{LongMoE} \\
    \midrule
    C       & $0.8455_{\pm.0021}$ & $0.8206_{\pm.0231}$ & $0.8166_{\pm.0423}$ & $0.8893_{\pm.0026}$ & $0.8636_{\pm.0006}$ & $\mathbf{0.8987_{\pm.0008}}^\dagger$ \\
    D       & $0.6892_{\pm.0036}$ & $0.6541_{\pm.0278}$ & $0.6866_{\pm.0017}$ & $0.7423_{\pm.0005}$ & $0.6692_{\pm.0016}$ & $\mathbf{0.8720_{\pm.0042}}^\dagger$ \\
    L       & $0.8150_{\pm.0027}$ & $0.8164_{\pm.0037}$ & $0.8245_{\pm.0024}$ & $0.8243_{\pm.0124}$ & $0.8380_{\pm.0008}$ & $\mathbf{0.8788_{\pm.0064}}^\dagger$ \\
    N       & $0.7780_{\pm.0040}$ & $0.7731_{\pm.0101}$ & $0.7732_{\pm.0035}$ & $0.7967_{\pm.0024}$ & $0.7980_{\pm.0011}$ & $\mathbf{0.8807_{\pm.0026}}^\dagger$ \\
    \midrule
    C,D     & $0.8556_{\pm.0016}$ & $0.8078_{\pm.0344}$ & $0.8267_{\pm.0446}$ & $0.8965_{\pm.0015}$ & $0.8710_{\pm.0010}$ & $\mathbf{0.9041_{\pm.0008}}^\dagger$ \\
    L,C     & $0.8813_{\pm.0013}$ & $0.8748_{\pm.0136}$ & $0.8760_{\pm.0116}$ & $0.9044_{\pm.0024}$ & $0.8846_{\pm.0007}$ & $\mathbf{0.9077_{\pm.0003}}^\dagger$ \\
    L,D     & $0.8335_{\pm.0018}$ & $0.8250_{\pm.0115}$ & $0.8273_{\pm.0178}$ & $0.8564_{\pm.0053}$ & $0.8511_{\pm.0009}$ & $\mathbf{0.8900_{\pm.0056}}^\dagger$ \\
    N,C     & $0.8789_{\pm.0008}$ & $0.8611_{\pm.0146}$ & $0.8644_{\pm.0162}$ & $0.9020_{\pm.0018}$ & $0.8887_{\pm.0003}$ & $\mathbf{0.9062_{\pm.0016}}^\dagger$ \\
    N,D     & $0.7957_{\pm.0027}$ & $0.7332_{\pm.0221}$ & $0.7935_{\pm.0025}$ & $0.8407_{\pm.0016}$ & $0.8044_{\pm.0009}$ & $\mathbf{0.8908_{\pm.0042}}^\dagger$ \\
    N,L     & $0.8690_{\pm.0019}$ & $0.8651_{\pm.0113}$ & $0.8680_{\pm.0020}$ & $0.8656_{\pm.0072}$ & $0.8793_{\pm.0008}$ & $\mathbf{0.8940_{\pm.0039}}^\dagger$ \\
    \midrule
    L,C,D   & $0.8874_{\pm.0005}$ & $0.8753_{\pm.0181}$ & $0.8782_{\pm.0162}$ & $0.9092_{\pm.0016}$ & $0.8901_{\pm.0009}$ & $\mathbf{0.9115_{\pm.0003}}^\dagger$ \\
    N,C,D   & $0.8850_{\pm.0006}$ & $0.8618_{\pm.0076}$ & $0.8729_{\pm.0165}$ & $0.9068_{\pm.0017}$ & $0.8913_{\pm.0006}$ & $\mathbf{0.9102_{\pm.0020}}^\dagger$ \\
    N,L,C   & $0.8984_{\pm.0007}$ & $0.8929_{\pm.0092}$ & $0.8960_{\pm.0049}$ & $0.9107_{\pm.0015}$ & $0.8998_{\pm.0002}$ & $\mathbf{0.9119_{\pm.0010}}^\dagger$ \\
    N,L,D   & $0.8778_{\pm.0013}$ & $0.8734_{\pm.0042}$ & $0.8711_{\pm.0083}$ & $0.8819_{\pm.0037}$ & $0.8834_{\pm.0006}$ & $\mathbf{0.9014_{\pm.0039}}^\dagger$ \\
    \midrule
    N,L,C,D & $0.9023_{\pm.0003}$ & $0.8993_{\pm.0082}$ & $0.8998_{\pm.0059}$ & $0.9043_{\pm.0014}$ & $0.9018_{\pm.0004}$ & $\mathbf{0.9150_{\pm.0012}}^\dagger$ \\
    \bottomrule
  \end{tabular}%
  }
\end{table}

\section{Results}
\label{sec:results}

\paragraph{Performance Analysis.}
\label{sec:performance}

Tables~\ref{tab:adni_perf}, \ref{tab:oasis_perf}, and~\ref{tab:mimic_perf} report AUC-ROC across all $2^L - 1$ modality combinations; $^\dagger$ denotes statistical significance (paired $t$-test, Bonferroni-corrected, $p < 0.05$, five seeds). On ADNI, \textsc{LongMoE} achieves the best AUC-ROC across all 15 combinations, with gains exceeding 0.45 in imaging- and genetics-only settings where single-visit baselines collapse (B: $0.9493$ vs.\ FuseMoE $0.4915$; G: $0.9517$ vs.\ $0.6316$), and maintains a near-flat AUC-ROC range of $[0.9493, 0.9755]$ including the full-modality setting ($0.9742$), reflecting the stability of longitudinal context across the full modality-combination space. On OASIS-3, \textsc{LongMoE} leads on four of seven combinations including full-modality (I,C,G: $0.8552$); Lee-MMGRU leads on clinical-heavy combinations (C: $0.8776$; C,G: $0.8674$; I,C: $0.8568$) where its recurrent encoder is competitive, but \textsc{LongMoE} recovers its advantage on imaging- and genetics-dominant combinations (G: $0.7624$; I: $0.7818$) where longitudinal trajectory context is the decisive signal. On MIMIC-IV, \textsc{LongMoE} achieves the best AUC-ROC across all 15 combinations (N,L,C,D: $0.9150$; D-only: $0.8720$ vs.\ MulT's $0.6892$ and TF's $0.6692$) with standard deviations consistently below $0.007$, confirming that the robustness mechanism transfers to a large-scale binary mortality task with fundamentally different modality semantics.
 
\paragraph{Complexity Analysis.}
\label{sec:complexity}

Table~\ref{tab:complexity} reports parameters, forward-pass time, GPU memory, training time, and AUC-ROC for all models on a single NVIDIA H100 GPU. \textsc{LongMoE} has 4.40M parameters versus 183K for Flex-MoE (24×) and a forward-pass latency of 6.04 ms versus 0.56 ms for TF (11×), a direct consequence of the trajectory-aware Transformer encoder processing the full visit sequence. GPU memory remains modest (149.8 MB on ADNI, 112.9 MB on MIMIC-IV, comparable to TF at 145.6 MB and 103.9 MB). The dominant cost is training rather than inference: 20.34 s/epoch on ADNI versus 2.47 s for TF. We report wall-clock time as the primary latency measure since GFLOPs underestimate actual latency due to memory access patterns in sequential Transformers; at 6 ms per patient the model sustains \textasciitilde165 patients per second on a single GPU, so full-cohort MIMIC-IV inference (\textasciitilde98K patients) takes \textasciitilde10 minutes and per-patient inference is well within interactive latency. The encoder scales as $\mathcal{O}(T^2 d)$ in visit sequence length; with $\bar{T} = 4.2$ (ADNI) and $\bar{T} = 3.1$ (MIMIC-IV) this cost is negligible, and linear-complexity attention variants~\citep{kitaev2020reformer} provide a direct scaling path for cohorts with longer sequences ($T > 50$).

\begin{table}[ht]
  \centering
  \caption{%
    Complexity comparison across models on all three datasets
    (full modality setting).
  }
  \label{tab:complexity}
  \renewcommand{\arraystretch}{1.25}
  \resizebox{0.85\textwidth}{!}{%
  \begin{tabular}{l l r r r r r r}
    \toprule
    \textbf{Dataset}
      & \textbf{Metric}
      & \textbf{MulT}
      & \textbf{Flex-MoE}
      & \textbf{FuseMoE}
      & \textbf{Lee-MMGRU}
      & \textbf{TF}
      & \textbf{LongMoE} \\
    \midrule
\multirow{5}{*}{\textbf{ADNI}}
 & \# Params $\downarrow$ & 640,771 & 183,683 & 580,248 & 116,175 & 109,699 & 4,396,305 \\
 & Fwd Time (ms) $\downarrow$ & 1.74 & 2.26 & 1.46 & 2.77 & \textbf{0.56} & 6.04 \\
 & GPU Mem (MB) $\downarrow$ & 114.6 & 109.0 & \textbf{103.9} & 123.7 & 145.6 & 149.8 \\
 & Train Time (s/ep) $\downarrow$ & 7.91 & 5.98 & 5.52 & 6.70 & \textbf{2.47} & 20.34 \\
 & AUC-ROC $\uparrow$ & 0.9666 & 0.9535 & 0.9604 & 0.9685 & 0.9320 & \textbf{0.9742} \\
\midrule
\multirow{5}{*}{\textbf{OASIS-3}}
 & \# Params $\downarrow$ & 623,363 & 145,283 & 427,032 & 78,540 & 92,675 & 4,261,262 \\
 & Fwd Time (ms) $\downarrow$ & 1.61 & 1.57 & 1.38 & 1.98 & \textbf{0.44} & 5.66 \\
 & GPU Mem (MB) $\downarrow$ & 103.0 & 105.2 & 101.0 & 116.2 & 140.2 & \textbf{61.0} \\
 & Train Time (s/ep) $\downarrow$ & 11.23 & 40.81 & 8.57 & 25.23 & \textbf{7.70} & 77.37 \\
 & AUC-ROC $\uparrow$ & 0.7713 & 0.8030 & 0.7811 & 0.8330 & 0.7562 & \textbf{0.8552} \\
\midrule
\multirow{5}{*}{\textbf{MIMIC-IV}}
 & \# Params $\downarrow$ & 662,530 & 205,442 & 601,104 & 149,258 & 131,458 & 1,321,614 \\
 & Fwd Time (ms) $\downarrow$ & 1.72 & 2.08 & 1.43 & 3.06 & \textbf{0.55} & 6.09 \\
 & GPU Mem (MB) $\downarrow$ & 93.7 & 87.9 & \textbf{82.7} & 103.3 & 103.9 & 112.9 \\
 & Train Time (s/ep) $\downarrow$ & 7.68 & 5.46 & 5.60 & 7.16 & \textbf{2.46} & 20.23 \\
 & AUC-ROC $\uparrow$ & 0.9023 & 0.8993 & 0.8998 & 0.9043 & 0.9018 & \textbf{0.9150} \\
\bottomrule
  \end{tabular}%
  }
\end{table}

\paragraph{Expert Specialisation Analysis.}

Figure~\ref{fig:expert_spec} analyses the routing behaviour of \textsc{LongMoE}'s $K=8$ experts on the ADNI testing set. The left plot shows a monotone disease-stage gradient across shared experts: E1 is predominantly CN, E2 is balanced CN/MCI, and E3/E4 increasingly concentrate MCI and AD patients, while individual experts collectively mirror the ADNI class prior ($40\%$ CN, $40\%$ MCI, $20\%$ AD) with distinct stage profiles, confirming that shared and individual experts capture complementary population-level and subgroup-specific patterns respectively. The right plot shows that shared experts are overwhelmingly activated by Biospecimen-containing combinations, whereas individual experts exhibit modality-aligned specialisation, confirming that context-conditioned routing learns meaningful specialisation along both disease-stage and modality-availability axes. 

\begin{figure}[htbp]
  \centering
  \includegraphics[width=0.65\textwidth]{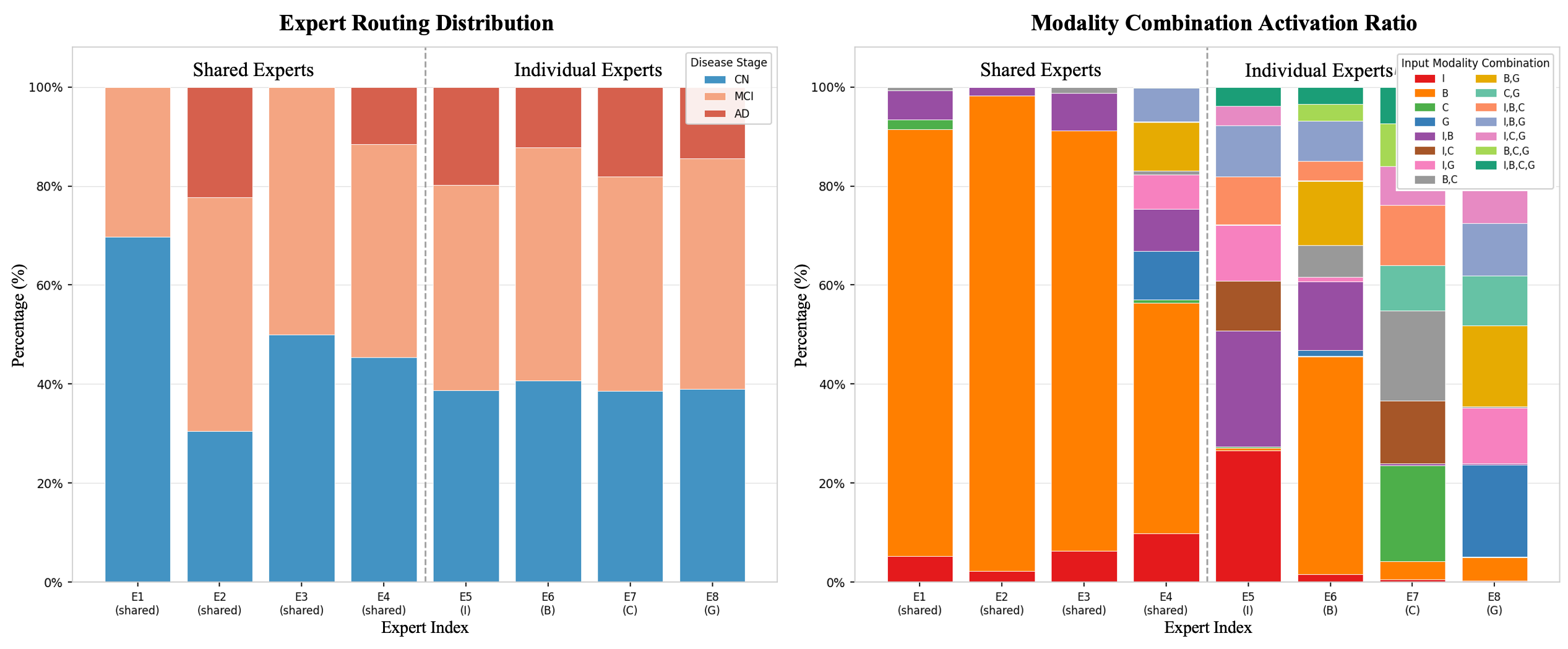}
  \caption{%
    \textit{Left:} Expert routing distribution for shared and individual experts.
    \textit{Right:} Modality combination activation ratio per expert.
  }
  \label{fig:expert_spec}
\end{figure}

\paragraph{Robustness under Missing Modalities.}
\label{sec:robustness}

Figure~\ref{fig:robustness} plots AUC-ROC as a function of missing modalities across all three benchmarks. On ADNI, \textsc{LongMoE} degrades minimally from 0.974 (all present) to 0.957 (one modality left; $\Delta\mathrm{AUC-ROC} = 0.017$), whereas every baseline loses between 0.20 and 0.33: the strongest baseline at maximum missingness, TF, falls from 0.932 to 0.730, and Lee-MMGRU from 0.968 to 0.637. On OASIS 3, \textsc{LongMoE} declines from 0.855 to 0.799, while MulT, Flex-MoE and FuseMoE fall to chance (0.51) once only one modality remains and Lee-MMGRU and TF reach 0.693 and 0.697. On MIMIC-IV, \textsc{LongMoE} retains 0.883 with a single modality (from 0.915), against 0.813 for Lee-MMGRU and 0.775-0.782 for the remaining baselines. Under our protocol a modality is withheld at the current visit, so \textsc{LongMoE} can still draw on earlier observations of it. Across all three benchmarks, the near-flat degradation curve confirms that the trajectory-aware encoder substitutes longitudinal context for
missing modality information.

\begin{figure}[ht]
  \centering
  \includegraphics[width=0.9\textwidth]{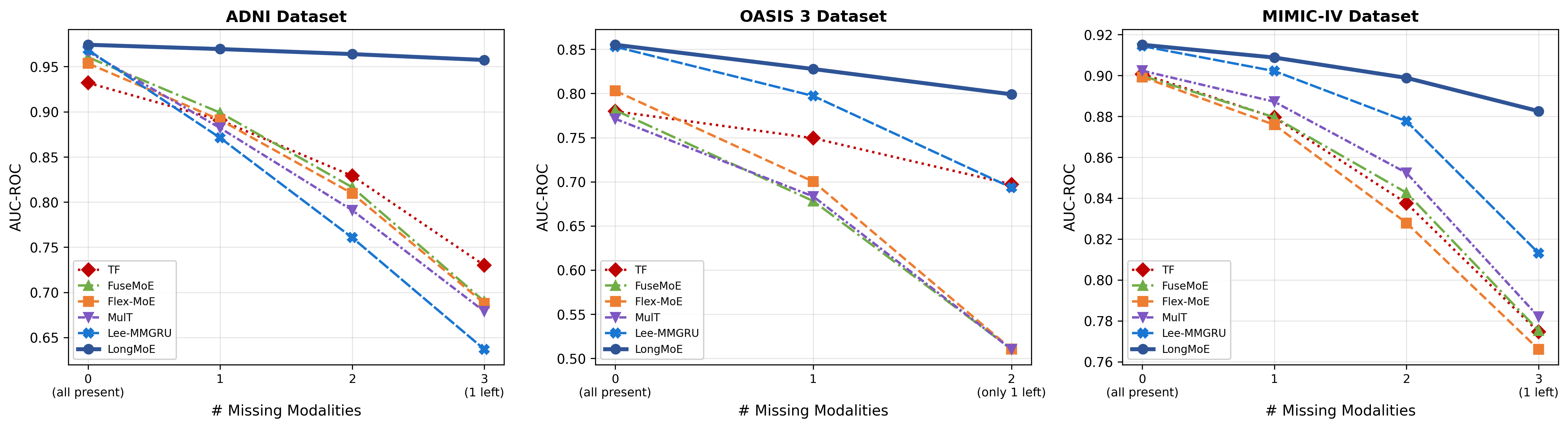}
  \caption{%
    Missing-modality robustness curves (AUC-ROC) for all the datasets.
  }
  \label{fig:robustness}
\end{figure}

\paragraph{Ablation Study.}
\label{sec:ablation}

\begin{wraptable}{r}{0.7\textwidth}
  \centering
  \vspace{-1em}
  \caption{%
     Component-wise ablation of \textsc{LongMoE} on ADNI dataset
    (full modality).
  }
  \label{tab:ablation}
  \small
  \renewcommand{\arraystretch}{1.3}
  \resizebox{\linewidth}{!}{%
  \begin{tabular}{l ccc ccc}
    \toprule
    \multirow{2}{*}{\textbf{Model Variants}}
      & \multicolumn{3}{c}{\textbf{Metrics}}
      & \multicolumn{3}{c}{\textbf{Change} $\Delta\downarrow$} \\
    \cmidrule(lr){2-4}\cmidrule(lr){5-7}
      & \textbf{Acc}
      & \textbf{AUC-ROC}
      & \textbf{F1}
      & $\Delta$\textbf{Acc}
      & $\Delta$\textbf{AUC}
      & $\Delta$\textbf{F1} \\
    \midrule
    \textit{w/o} LE
      & 0.8452 & 0.9407 & 0.7250
      & $-0.0329$ & $-0.0335$ & $-0.0814$ \\
    \textit{w/o} VE
      & 0.8556 & 0.9564 & 0.7816
      & $-0.0225$ & $-0.0178$ & $-0.0248$ \\
    \textit{w/o} Imp-ZP
      & 0.8650 & 0.9607 & 0.7850
      & $-0.0131$ & $-0.0135$ & $-0.0214$ \\
    \textit{w/o} Imp-GM
      & 0.8612 & 0.9631 & 0.7910
      & $-0.0169$ & $-0.0111$ & $-0.0154$ \\
    \textit{w/o} MoE
      & 0.8405 & 0.9501 & 0.7824
      & $-0.0377$ & $-0.0241$ & $-0.0240$ \\
    \textit{w/o} Shared Expert
      & 0.8631 & 0.9606 & 0.7895
      & $-0.0150$ & $-0.0136$ & $-0.0169$ \\
    \textit{w/o} $\mathcal{L}_{\mathrm{bal}}$
      & 0.8650 & 0.9547 & 0.7904
      & $-0.0131$ & $-0.0195$ & $-0.0160$ \\
    \midrule
    \textbf{LongMoE}
      & \textbf{0.8781} & \textbf{0.9742} & \textbf{0.8064}
      & --- & --- & --- \\
    \bottomrule
  \end{tabular}%
  }
  \vspace{-0.5em}
\end{wraptable}

Table~\ref{tab:ablation} reports a component-wise ablation on ADNI (full modality), confirming every component contributes positively. Replacing the MoE with a single expert produces drop ($\Delta\mathrm{AUC}{=}{-0.0241}$), confirming expert specialisation provides benefit. Removing the Longitudinal Encoder (LE; $\Delta\mathrm{Acc}{=}{-0.0329}$, $\Delta\mathrm{AUC}{=}{-0.0335}$) produces drop, which reduces \textsc{LongMoE} to a per-visit classifier, and Viscode Encoding (VE; $\Delta\mathrm{Acc}{=}{-0.0225}$, $\Delta\mathrm{AUC}{=}{-0.0178}$), validating the sinusoidal temporal encoding as an independent longitudinal signal beyond self-attention. Load Balancing ($\mathcal{L}_{\mathrm{bal}}$) produces drop ($\Delta\mathrm{AUC}{=}{-0.0195}$), both imputation ablations degrade comparably with Zero Padding slightly worse than Global Mean (greater distributional shift). The detailed ablation studies are provided in Appendix~\ref{app:ablation_masking}.

\section{Conclusion}
\label{sec:conclusion}
 
We presented \textsc{LongMoE}, a unified framework that jointly addresses modality missingness and longitudinal temporal dynamics in multimodal sequential learning. By composing a context-aware imputation module, attentional tokenization module, and a trajectory-aware Transformer encoder with a context-conditioned Sparse MoE routing layer, \textsc{LongMoE} enables patient-specific conditional computation over the full visit history under arbitrary modality availability. Experiments on different datasets show that \textsc{LongMoE} is strongest under missing or weak contemporaneous modalities, but clinical-score-heavy settings remain competitive for recurrent baselines.

\section*{AI Use Statement}
Generative AI tools were used solely for language editing of the manuscript, including improving grammar and readability of author-written text. They were not used to generate research ideas, experimental results, or the statements and proofs of the theoretical results, all of which were produced by the authors. All the work was reviewed by the authors, who take full responsibility for the final content of this work, including all text, claims, and artifacts.

\section*{Ethics Statement}
This work uses only existing, de-identified datasets (ADNI, OASIS-3, and MIMIC-IV) accessed under their respective data use agreements, with MIMIC-IV obtained through PhysioNet credentialed access after completion of the required human subjects research training. No new data were collected, no human participants were contacted, and no attempt was made to re-identify individuals. We acknowledge that these cohorts are not fully representative of the general population: ADNI and OASIS-3 participants are predominantly highly educated and of European ancestry, and MIMIC-IV reflects a single academic medical centre in the United States. Model performance may therefore not transfer equitably across demographic groups or care settings. \textsc{LongMoE} is a research prototype intended to support rather than replace clinical judgement, and it should not be used for diagnosis or prognosis without prospective validation, subgroup fairness evaluation, and appropriate regulatory review. The authors declare no competing interests.

\section*{Reproducibility Statement}
We have made several efforts to ensure reproducibility. The full model architecture is specified in Section~\ref{sec:architecture}, and the experimental protocol, including patient-level stratified 70/15/15 splits, final-visit evaluation, and the current-visit masking scheme, is described in Section~\ref{sec:experiments}. Dataset statistics and preprocessing steps are provided in Section~\ref{sec:experiments} and Appendix~\ref{sec:datasets}, and complete hyperparameter configurations for \textsc{LongMoE} and all baselines are listed in Appendix~\ref{app:hyperparams} (Tables~\ref{tab:hp_longmoe} to~\ref{tab:hp_tf}). All results are reported as mean and standard deviation over five random seeds with Bonferroni-corrected significance tests. Statements and proofs of all theoretical results appear in Appendices~\ref{app:theory_viscode} to~\ref{app:theory_convergence}. Anonymised source code, including preprocessing scripts, configuration files, and random seeds, is provided as supplementary material and will be released publicly upon acceptance. All datasets are available to qualified researchers from their original providers.

\bibliography{iclr2027_conference}

@article{baltrusaitis2018multimodal,
  author  = {Baltru{\v{s}}aitis, T. and Ahuja, C. and Morency, L.-P.},
  title   = {Multimodal Machine Learning: A Survey and Taxonomy},
  journal = {IEEE Transactions on Pattern Analysis and Machine Intelligence},
  volume  = {41},
  number  = {2},
  pages   = {423--443},
  year    = {2018}
}

@inproceedings{liang2021multibench,
  author    = {Liang, P. P. and Lyu, Y. and Fan, X. and Wu, Z. and Cheng, Y. and Wu, J. and Chen, L. and Wu, P. and Lee, M. A. and Zhu, Y. and Salakhutdinov, R. and Morency, L.},
  title     = {{MultiBench}: Multiscale Benchmarks for Multimodal Representation Learning},
  booktitle = {Advances in Neural Information Processing Systems (NeurIPS)},
  volume    = {34},
  pages     = {5813--5826},
  year      = {2021}
}

@article{jack2013biomarker,
  author  = {Jack, C. R. and Holtzman, D. M.},
  title   = {Biomarker Modeling of {Alzheimer's} Disease},
  journal = {Neuron},
  volume  = {80},
  number  = {6},
  pages   = {1347--1358},
  year    = {2013}
}

@article{marquez2019neuroimaging,
  author  = {M{\'a}rquez, F. and Yassa, M. A.},
  title   = {Neuroimaging Biomarkers for {Alzheimer's} Disease},
  journal = {Molecular Neurodegeneration},
  volume  = {14},
  number  = {1},
  pages   = {21},
  year    = {2019}
}

@article{papassotiropoulos2006genetics,
  author  = {Papassotiropoulos, A. and Fountoulakis, M. and Dunckley, T. and Stephan, D. A. and Reiman, E. M.},
  title   = {Genetics, Transcriptomics, and Proteomics of {Alzheimer's} Disease},
  journal = {Journal of Clinical Psychiatry},
  volume  = {67},
  number  = {4},
  pages   = {652},
  year    = {2006}
}

@inproceedings{yun2024flexmoe,
  author    = {Yun, S. and Choi, I. and Peng, J. and Wu, Y. and Bao, J. and Zhang, Q. and Xin, J. and Long, Q. and Chen, T.},
  title     = {Flex-{MoE}: Modeling Arbitrary Modality Combination via the Flexible Mixture-of-Experts},
  booktitle = {Advances in Neural Information Processing Systems (NeurIPS)},
  year      = {2024}
}

@inproceedings{han2024fusemoe,
  author    = {Han, X. and Nguyen, H. and Harris, C. and Ho, N. and Saria, S.},
  title     = {{FuseMoE}: Mixture-of-Experts Transformers for Fleximodal Fusion},
  booktitle = {Advances in Neural Information Processing Systems (NeurIPS)},
  year      = {2024}
}

@article{zhang2022mmformer,
  author  = {Zhang, Y. and He, N. and Yang, J. and Li, Y. and Wei, D. and Huang, Y. and Zhang, Y. and He, Z. and Zheng, Y.},
  title   = {{mmFormer}: Multimodal Medical Transformer for Incomplete Multimodal Learning of Brain Tumor Segmentation},
  journal = {arXiv preprint arXiv:2206.02425},
  year    = {2022}
}

@inproceedings{wang2024sharedspecific,
  author    = {Wang, H. and Chen, Y. and Ma, C. and Avery, J. and Hull, L. and Carneiro, G.},
  title     = {Multi-modal Learning with Missing Modality via Shared-Specific Feature Modelling},
  booktitle = {Proceedings of the IEEE/CVF Conference on Computer Vision and Pattern Recognition (CVPR)},
  year      = {2024}
}

@inproceedings{shazeer2017outrageously,
  author    = {Shazeer, N. and Mirhoseini, A. and Maziarz, K. and Davis, A. and Le, Q. and Hinton, G. and Dean, J.},
  title     = {Outrageously Large Neural Networks: The Sparsely-Gated Mixture-of-Experts Layer},
  booktitle = {International Conference on Learning Representations (ICLR)},
  year      = {2017}
}

@article{fedus2022switch,
  author  = {Fedus, W. and Zoph, B. and Shazeer, N.},
  title   = {Switch Transformers: Scaling to Trillion Parameter Models with Simple and Efficient Sparsity},
  journal = {Journal of Machine Learning Research},
  volume  = {23},
  number  = {120},
  pages   = {1--39},
  year    = {2022}
}

@article{jiang2024mixtral,
  author  = {Jiang, A. Q. and Sablayrolles, A. and Roux, A. and Mensch, A. and Savary, B. and Bamford, C. and Chaplot, D. S. and de las Casas, D. and Hanna, E. B. and Bressand, F. and others},
  title   = {Mixtral of Experts},
  journal = {arXiv preprint arXiv:2401.04088},
  year    = {2024}
}

@inproceedings{mustafa2022limoe,
  author    = {Mustafa, B. and Riquelme, C. and Puigcerver, J. and Jenatton, R. and Houlsby, N.},
  title     = {Multimodal Contrastive Learning with {LiMoE}: The Language-Image Mixture of Experts},
  booktitle = {Advances in Neural Information Processing Systems (NeurIPS)},
  year      = {2022}
}

@inproceedings{peng2024sparse,
  author    = {Peng, J. and Zhou, K. and Zhou, R. and Hartvigsen, T. and Zhang, Y. and Wang, Z. and Chen, T.},
  title     = {Sparse {MoE} as a New Treatment: Addressing Forgetting, Fitting, and Learning Issues in Multi-modal Multi-task Learning},
  booktitle = {International Conference on Learning Representations (ICLR)},
  year      = {2024}
}

@inproceedings{zadeh2017tensor,
  author    = {Zadeh, A. and Chen, M. and Poria, S. and Cambria, E. and Morency, L.-P.},
  title     = {Tensor Fusion Network for Multimodal Sentiment Analysis},
  booktitle = {Proceedings of the Conference on Empirical Methods in Natural Language Processing (EMNLP)},
  pages     = {1103--1114},
  year      = {2017}
}

@inproceedings{tsai2019multimodal,
  author    = {Tsai, Y.-H. H. and Bai, S. and Liang, P. P. and Kolter, J. Z. and Morency, L.-P. and Salakhutdinov, R.},
  title     = {Multimodal Transformer for Unaligned Multimodal Language Sequences},
  booktitle = {Proceedings of the Annual Meeting of the Association for Computational Linguistics (ACL)},
  pages     = {6558--6569},
  year      = {2019}
}

@article{lee2019predicting,
  author  = {Lee, G. and Nho, K. and Kang, B. and Sohn, K.-A. and Kim, D.},
  title   = {Predicting {Alzheimer's} Disease Progression Using Multi-modal Deep Learning Approach},
  journal = {Scientific Reports},
  volume  = {9},
  number  = {1},
  pages   = {1952},
  year    = {2019}
}

@article{venugopalan2021multimodal,
  author  = {Venugopalan, J. and Tong, L. and Hassanzadeh, H. R. and Wang, M. D.},
  title   = {Multimodal Deep Learning Models for Early Detection of {Alzheimer's} Disease Stage},
  journal = {Scientific Reports},
  volume  = {11},
  number  = {1},
  pages   = {3254},
  year    = {2021}
}

@article{odusami2023machine,
  author  = {Odusami, M. and Maskeliunas, R. and Damasevicius, R. and Misra, S.},
  title   = {Machine Learning with Multi-modal Neuroimaging Data to Classify Stages of {Alzheimer's} Disease: A Systematic Review and Meta-analysis},
  journal = {Cognitive Neurodynamics},
  pages   = {1--20},
  year    = {2023}
}

@inproceedings{vaswani2017attention,
  author    = {Vaswani, A. and Shazeer, N. and Parmar, N. and Uszkoreit, J. and Jones, L. and Gomez, A. N. and Kaiser, {\L}. and Polosukhin, I.},
  title     = {Attention Is All You Need},
  booktitle = {Advances in Neural Information Processing Systems (NeurIPS)},
  volume    = {30},
  pages     = {5998--6008},
  year      = {2017}
}

@inproceedings{devlin2019bert,
  author    = {Devlin, J. and Chang, M.-W. and Lee, K. and Toutanova, K.},
  title     = {{BERT}: Pre-training of Deep Bidirectional Transformers for Language Understanding},
  booktitle = {Proceedings of the Annual Conference of the North American Chapter of the Association for Computational Linguistics (NAACL)},
  pages     = {4171--4186},
  year      = {2019}
}

@article{weiner2010adni,
  author  = {Weiner, M. W. and Aisen, P. S. and Jack, C. R. and Jagust, W. J. and Trojanowski, J. Q. and Shaw, L. and Saykin, A. J. and Morris, J. C. and Cairns, N. and Beckett, L. A. and others},
  title   = {The {Alzheimer's} Disease Neuroimaging Initiative: Progress Report and Future Plans},
  journal = {Alzheimer's \& Dementia},
  volume  = {6},
  number  = {3},
  pages   = {202--211},
  year    = {2010}
}

@article{lamontagne2019oasis,
  author  = {LaMontagne, P. J. and Benzinger, T. L. and Morris, J. C. and Keefe, S. and Hornbeck, R. and Yao, C. and Hintaus, A. and Marcus, D.},
  title   = {{OASIS-3}: Longitudinal Neuroimaging, Clinical, and Cognitive Dataset for Normal Aging and {Alzheimer's} Disease},
  journal = {medRxiv},
  year    = {2019}
}

@article{johnson2020mimic,
  author  = {Johnson, A. and Bulgarelli, L. and Pollard, T. and Horng, S. and Celi, L. A. and Mark, R.},
  title   = {{MIMIC-IV}, a Freely Accessible Electronic Health Record Dataset},
  journal = {Scientific Data},
  volume  = {10},
  number  = {1},
  pages   = {1--9},
  publisher = {Nature Publishing Group},
  year    = {2023}
}

@article{doshi2016muse,
  author  = {Doshi, J. and Erus, G. and Ou, Y. and Resnick, S. M. and Gur, R. C. and Gur, R. E. and Satterthwaite, T. D. and Furth, S. and Davatzikos, C. and Initiative, A. D. N.},
  title   = {{MUSE}: Multi-atlas Region Segmentation Utilizing Ensembles of Registration Algorithms and Parameters, and Locally Optimal Atlas Selection},
  journal = {NeuroImage},
  volume  = {127},
  pages   = {186--195},
  year    = {2016}
}

@article{ou2011dramms,
  author  = {Ou, Y. and Sotiras, A. and Paragios, N. and Davatzikos, C.},
  title   = {{DRAMMS}: Deformable Registration via Attribute Matching and Mutual-Saliency Weighting},
  journal = {Medical Image Analysis},
  volume  = {15},
  number  = {4},
  pages   = {622--639},
  year    = {2011}
}

@inproceedings{loshchilov2019decoupled,
  author    = {Loshchilov, I. and Hutter, F.},
  title     = {Decoupled Weight Decay Regularization},
  booktitle = {International Conference on Learning Representations (ICLR)},
  year      = {2019}
}

@book{little2019statistical,
  author    = {Little, R. J. A. and Rubin, D. B.},
  title     = {Statistical Analysis with Missing Data},
  edition   = {3rd},
  publisher = {John Wiley \& Sons},
  year      = {2019}
}

@inproceedings{riquelme2021scaling,
  author    = {Riquelme, C. and Puigcerver, J. and Mustafa, B. and Neumann, M. and Jenatton, R. and Pinto, A. S. and Keysers, D. and Houlsby, N.},
  title     = {Scaling Vision with Sparse Mixture of Experts},
  booktitle = {Advances in Neural Information Processing Systems (NeurIPS)},
  volume    = {34},
  pages     = {8583--8595},
  year      = {2021}
}

@inproceedings{ma2018modeling,
  author    = {Ma, J. and Zhao, Z. and Yi, X. and Chen, J. and Hong, L. and Chi, E. H.},
  title     = {Modeling Task Relationships in Multi-task Learning with Multi-gate Mixture-of-Experts},
  booktitle = {Proceedings of the ACM SIGKDD International Conference on Knowledge Discovery and Data Mining (KDD)},
  pages     = {1930--1939},
  year      = {2018}
}

@inproceedings{chen2023modsquad,
  author    = {Chen, Z. and Shen, Y. and Ding, M. and Chen, Z. and Zhao, H. and Learned-Miller, E. G. and Gan, C.},
  title     = {Mod-Squad: Designing Mixtures of Experts as Modular Multi-task Learners},
  booktitle = {Proceedings of the IEEE/CVF Conference on Computer Vision and Pattern Recognition (CVPR)},
  pages     = {11828--11837},
  year      = {2023}
}

@inproceedings{zhou2022expertchoice,
  author    = {Zhou, Y. and Lei, T. and Liu, H. and Du, N. and Huang, Y. and Zhao, V. and Dai, A. and Chen, Z. and Le, Q. and Laudon, J.},
  title     = {Mixture-of-Experts with Expert Choice Routing},
  booktitle = {Advances in Neural Information Processing Systems (NeurIPS)},
  year      = {2022}
}

@inproceedings{hazimeh2021dselectk,
  author    = {Hazimeh, H. and Zhao, Z. and Chowdhery, A. and Sathiamoorthy, M. and Chen, Y. and Mazumder, R. and Hong, L. and Chi, E. H.},
  title     = {{DSelect-$k$}: Differentiable Selection in the Mixture of Experts with Applications to Multi-task Learning},
  booktitle = {Advances in Neural Information Processing Systems (NeurIPS)},
  volume    = {34},
  pages     = {29335--29347},
  year      = {2021}
}

@article{hochreiter1997lstm,
  author  = {Hochreiter, S. and Schmidhuber, J.},
  title   = {Long Short-Term Memory},
  journal = {Neural Computation},
  volume  = {9},
  number  = {8},
  pages   = {1735--1780},
  year    = {1997}
}

@inproceedings{cho2014gru,
  author    = {Cho, K. and van Merri{\"e}nboer, B. and Gulcehre, C. and Bahdanau, D. and Bougares, F. and Schwenk, H. and Bengio, Y.},
  title     = {Learning Phrase Representations Using {RNN} Encoder--Decoder for Statistical Machine Translation},
  booktitle = {Proceedings of the Conference on Empirical Methods in Natural Language Processing (EMNLP)},
  pages     = {1724--1734},
  year      = {2014}
}

@inproceedings{choi2016retain,
  author    = {Choi, E. and Bahadori, M. T. and Sun, J. and Kulas, J. and Schuetz, A. and Stewart, W.},
  title     = {{RETAIN}: An Interpretable Predictive Model for Healthcare Using Reverse Time Attention Mechanism},
  booktitle = {Advances in Neural Information Processing Systems (NeurIPS)},
  volume    = {29},
  pages     = {3504--3512},
  year      = {2016}
}

@inproceedings{horn2020sett,
  author    = {Horn, M. and Moor, M. and Bock, C. and Rieck, B. and Borgwardt, K.},
  title     = {Set Functions for Time Series},
  booktitle = {Proceedings of the International Conference on Machine Learning (ICML)},
  volume    = {119},
  pages     = {4353--4363},
  year      = {2020}
}

@inproceedings{shukla2021multitime,
  author    = {Shukla, N. and Marlin, B.},
  title     = {Multi-time Attention Networks for Irregularly Sampled Time Series},
  booktitle = {International Conference on Learning Representations (ICLR)},
  year      = {2021}
}

@article{jordan1994hierarchical,
  author  = {Jordan, M. I. and Jacobs, R. A.},
  title   = {Hierarchical Mixtures of Experts and the {EM} Algorithm},
  journal = {Neural Computation},
  volume  = {6},
  number  = {2},
  pages   = {181--214},
  year    = {1994}
}

@article{jacobs1991adaptive,
  author  = {Jacobs, R. A. and Jordan, M. I. and Nowlan, S. J. and Hinton, G. E.},
  title   = {Adaptive Mixtures of Local Experts},
  journal = {Neural Computation},
  volume  = {3},
  number  = {1},
  pages   = {79--87},
  year    = {1991}
}

@article{kiela2019supervised,
  author  = {Kiela, D. and Bhooshan, C. and Firooz, H. and Perez, E. and Testuggine, D.},
  title   = {Supervised Multimodal Bitransformers for Classifying Images and Text},
  journal = {arXiv preprint arXiv:1909.02950},
  year    = {2019}
}

@inproceedings{kitaev2020reformer,
  author    = {Kitaev, N. and Kaiser, L. and Levskaya, A.},
  title     = {Reformer: The Efficient Transformer},
  booktitle = {International Conference on Learning Representations (ICLR)},
  year      = {2020}
}

@book{cover2006elements,
  author    = {Cover, T. M. and Thomas, J. A.},
  title     = {Elements of Information Theory},
  edition   = {2nd},
  publisher = {John Wiley \& Sons},
  year      = {2006}
}

@article{hornik1989universality,
  author  = {Hornik, K. and Stinchcombe, M. and White, H.},
  title   = {Multilayer Feedforward Networks Are Universal Approximators},
  journal = {Neural Networks},
  volume  = {2},
  number  = {5},
  pages   = {359--366},
  year    = {1989}
}

@article{bottou2018optimization,
  author  = {Bottou, L. and Curtis, F. E. and Nocedal, J.},
  title   = {Optimization Methods for Large-Scale Machine Learning},
  journal = {SIAM Review},
  volume  = {60},
  number  = {2},
  pages   = {223--311},
  year    = {2018}
}

@inproceedings{rahman2023modelling,
  author    = {Rahman, M. R. and Hussain, M. and Piper, T. and Geyer, H. and Equey, T. and Baume, N. and others},
  title     = {Modelling Metabolism Pathways Using Graph Representation Learning for Fraud Detection in Sports},
  booktitle = {2023 IEEE International Conference on Digital Health (ICDH)},
  pages     = {158--168},
  year      = {2023},
  publisher = {IEEE}
}

@misc{rahman2022data,
  author = {Rahman, M. R. and Piper, T. and Geyer, H. and Equey, T. and Baume, N. and Aikin, R. and Maass, W.},
  title  = {Data Analytics for Uncovering Fraudulent Behaviour in Elite Sports},
  year   = {2022}
}

@article{john2024llmrs,
  author  = {John, A. and Aidoo, T. and Behmanush, H. and Gunduz, I. B. and Shrestha, H. and Rahman, M. R. and Maaß, W.},
  title   = {{LLMRS}: Unlocking Potentials of {LLM}-Based Recommender Systems for Software Purchase},
  journal = {arXiv preprint arXiv:2401.06676},
  year    = {2024}
}

@inproceedings{rahman2024sacnn,
  author    = {Rahman, M. R. and Khaliq, L. A. and Piper, T. and Geyer, H. and Equey, T. and Baume, N. and others},
  title     = {{SACNN}: Self Attention-Based Convolutional Neural Network for Fraudulent Behaviour Detection in Sports},
  booktitle = {Proceedings of the Thirty-Third International Joint Conference on Artificial Intelligence},
  pages     = {6017--6025},
  year      = {2024},
  month     = aug
}
\clearpage

\appendix
\section{Appendix}

\label{app:theory}
 
We provide theoretical grounding for the four core
architectural components of \textsc{LongMoE}: the multi-frequency
sinusoidal viscode encoding, the context-aware imputation module,
the trajectory-aware Transformer encoder, and the sparse MoE routing
layer.
 
\subsection{Multi-Frequency Sinusoidal Viscode Encoding}
\label{app:theory_viscode}
 
\begin{proposition}[Injectivity of Multi-Frequency Sinusoidal
Encoding]
\label{prop:viscode_injectivity}
Let $\phi: \mathbb{Z}_{>0} \rightarrow \mathbb{R}^{2F}$ be the
multi-frequency sinusoidal encoding defined by:
\begin{equation}
  \phi(\tau)
  = \bigl[\sin(\omega_1\tau),\,\cos(\omega_1\tau),\,\ldots,\,
           \sin(\omega_F\tau),\,\cos(\omega_F\tau)\bigr],
  \label{eq:phi_def}
\end{equation}
where $\omega_k = 2^{k-1}$ for $k = 1, \ldots, F$.
Then $\phi$ is injective on all of $\mathbb{Z}_{>0}$: for any two
distinct integer timestamps $\tau, \tau' \in \mathbb{Z}_{>0}$ with
$\tau \neq \tau'$, we have $\phi(\tau) \neq \phi(\tau')$.
In particular, $\phi$ is injective on the clinically relevant range
$\tau \in \{1, 2, \ldots, 240\}$ (months 1-240, covering a 20-year
longitudinal follow-up window encompassing all ADNI, OASIS-3, and
MIMIC-IV visit ranges).
\end{proposition}
 
\begin{proof}
Suppose $\phi(\tau) = \phi(\tau')$ for some
$\tau, \tau' \in \mathbb{Z}_{>0}$.
This requires $\sin(\omega_k\tau) = \sin(\omega_k\tau')$ and
$\cos(\omega_k\tau) = \cos(\omega_k\tau')$ simultaneously for all
$k = 1, \ldots, F$, which holds if and only if:
\begin{equation}
  \omega_k(\tau - \tau') \equiv 0 \pmod{2\pi}
  \quad \forall\, k = 1, \ldots, F.
  \label{eq:collision_cond}
\end{equation}
Condition~\eqref{eq:collision_cond} is equivalent to requiring
$\tau - \tau' = 2\pi m_k / \omega_k = 2\pi m_k / 2^{k-1}$ for
some integer $m_k$ for each $k = 1, \ldots, F$.
 
We now show this forces $\tau = \tau'$.
Since $\tau - \tau' \in \mathbb{Z}$ by assumption, we need
$2\pi m_k / 2^{k-1} \in \mathbb{Z}$ for each $k$.
This requires $2\pi \in \mathbb{Q}$, since $m_k / 2^{k-1}$
is rational for all integer $m_k$ and all $k$.
However, $2\pi$ is transcendental (in particular, irrational),
so $2\pi m_k / 2^{k-1} \notin \mathbb{Z}$ for any $m_k \neq 0$.
Therefore the only solution to~\eqref{eq:collision_cond} is
$m_k = 0$ for all $k$, which gives $\tau - \tau' = 0$,
i.e., $\tau = \tau'$.
Hence $\phi(\tau) = \phi(\tau')$ implies $\tau = \tau'$, and
$\phi$ is injective on all of $\mathbb{Z}_{>0}$.
 
The clinical range $\{1, \ldots, 240\}$ is a subset of
$\mathbb{Z}_{>0}$, so injectivity holds over this range as an
immediate corollary.
\end{proof}
 
\begin{remark}
The key step in the proof is the irrationality of $2\pi$, which
prevents any non-zero integer difference $\tau - \tau'$ from
satisfying the collision condition~\eqref{eq:collision_cond}.
This argument holds for \emph{any} choice of rational or integer
base frequencies $\omega_k > 0$, and is not restricted to the
geometric schedule $\omega_k = 2^{k-1}$. The geometric schedule
is chosen to provide a wide dynamic range of temporal scales
(from monthly to decadal patterns) rather than for the injectivity
guarantee, which holds for any such frequencies.
Concretely, with $F = 8$ frequency bands and timestamps in
$\{1, \ldots, 240\}$, every pair of visits is guaranteed to
receive a distinct 16-dimensional encoding, providing a unique
temporal identity for all longitudinal sequences within the
clinically plausible follow-up window of 20 years.
\end{remark}
 
\begin{proposition}[Expressivity of Multi-Frequency Encoding]
\label{prop:viscode_expressivity}
Let $f: \mathbb{Z}_{>0} \rightarrow \mathbb{R}$ be any function
defined on a finite set of integer timestamps
$\mathcal{T} = \{\tau_1, \ldots, \tau_T\}$.
Then there exists an MLP $g: \mathbb{R}^{2F} \rightarrow \mathbb{R}$
with $F \geq T$ such that $g(\phi(\tau_t)) = f(\tau_t)$ for all
$t = 1, \ldots, T$.
\end{proposition}
 
\begin{proof}
By Proposition~\ref{prop:viscode_injectivity}, the matrix
$\Phi \in \mathbb{R}^{T \times 2F}$ with rows $\phi(\tau_t)$ has
distinct rows.
For $F \geq T$, the $2F$-dimensional encoding space can
be mapped to an arbitrary target vector $[f(\tau_1), \ldots,
f(\tau_T)]^\top$ by solving the linear system $\Phi \mathbf{w} =
\mathbf{f}$, which has a solution when $\Phi$ has full row rank.
Since $\phi$ is injective and the frequencies $\omega_k$ are
distinct, $\Phi$ constitutes a generalised Vandermonde-type matrix
with full row rank for $T \leq 2F$, guaranteeing the existence of
such a mapping.
A single-layer linear MLP with weight vector $\mathbf{w}$ achieves
this mapping exactly; a multi-layer MLP achieves it with strictly
weaker conditions on $F$.
\end{proof}

\subsection{Context-Aware Imputation}
\label{app:theory_imputation}
 
\begin{proposition}[Optimality of Carry-Forward Imputation under
Short-Term Autocorrelation]
\label{prop:carryforward}
Let $\{e^{(l)}_t\}_{t \geq 1}$ be a stationary stochastic process
with autocorrelation function $\rho(h) = \mathrm{Corr}(e^{(l)}_t,
e^{(l)}_{t+h})$, and let $\hat{e}^{(l)}_t$ denote the minimum
mean-squared error (MMSE) linear estimator of a missing embedding
at time $t$ given the last observed embedding at time $t-1$.
If $|\rho(1)| \geq |\rho(h)|$ for all $h > 1$ (i.e., the process
exhibits monotonically decaying autocorrelation), then the
carry-forward estimator $\hat{e}^{(l)}_t = e^{(l)}_{t-1}$
minimises the expected imputation error:
\begin{equation}
  \mathbb{E}\!\left[\left\|e^{(l)}_t - \hat{e}^{(l)}_t
  \right\|^2\right]
  \leq \mathbb{E}\!\left[\left\|e^{(l)}_t - \hat{e}^{(l,\mathrm{alt})}_t
  \right\|^2\right]
\end{equation}
for any linear estimator $\hat{e}^{(l,\mathrm{alt})}_t$ that
depends on a single past observation at time $t-k$ with $k > 1$.
\end{proposition}
 
\begin{proof}
For a zero-mean stationary process, the MMSE linear estimator of
$e^{(l)}_t$ given $e^{(l)}_{t-k}$ is
$\hat{e}^{(l)}_t = \rho(k) \cdot e^{(l)}_{t-k}$.
The corresponding MSE is:
\begin{equation}
  \mathrm{MSE}(k) = \mathbb{E}\!\left[\|e^{(l)}_t\|^2\right]
  \left(1 - \rho(k)^2\right).
\end{equation}
Since $|\rho(1)| \geq |\rho(k)|$ for all $k > 1$ by assumption,
$1 - \rho(1)^2 \leq 1 - \rho(k)^2$, hence $\mathrm{MSE}(1) \leq
\mathrm{MSE}(k)$.
The carry-forward estimator corresponds to $k = 1$, establishing
its MMSE optimality among single-lag linear estimators.
\end{proof}
 
\begin{remark}
Biomarker embedding processes in longitudinal AD datasets exhibit
strong short-term autocorrelation, e.g., hippocampal volume changes
slowly relative to visit frequency, making the monotone decay
assumption empirically reasonable.
The proposition formalises the clinical intuition that the most
recent observation is the most informative single prior for
imputation.
\end{remark}
 
\subsection{Trajectory-Aware Longitudinal Encoder}
\label{app:theory_encoder}
 
\begin{theorem}[Trajectory-Aware Fusion Dominates Per-Visit Fusion
under Markovian Disease Dynamics]
\label{thm:traj_dominates}
Let the patient's disease state at visit $t$ be represented by a
latent variable $s_t \in \mathcal{S}$, and let the observed
multimodal embedding $\mathbf{z}_t \in \mathbb{R}^d$ be a noisy
function of $s_t$: $\mathbf{z}_t = f(s_t) + \epsilon_t$, where
$\epsilon_t$ is i.i.d.\ mean-zero noise with covariance
$\sigma^2 \mathbf{I}$.
Assume disease progression follows a first-order Markov chain:
$p(s_t \mid s_1, \ldots, s_{t-1}) = p(s_t \mid s_{t-1})$.
Let $\hat{y}^{\mathrm{per}}$ denote the prediction from a
per-visit model with access only to $\mathbf{z}_t$, and
$\hat{y}^{\mathrm{traj}}$ denote the prediction from a
trajectory-aware model with access to $(\mathbf{z}_1, \ldots,
\mathbf{z}_t)$.
Then:
\begin{equation}
  \mathbb{E}\!\left[\mathcal{L}(y_t, \hat{y}^{\mathrm{traj}})\right]
  \leq
  \mathbb{E}\!\left[\mathcal{L}(y_t, \hat{y}^{\mathrm{per}})\right],
\end{equation}
with strict inequality whenever $\mathbb{I}(y_t; s_{t-1} \mid
\mathbf{z}_t) > 0$, i.e., when past states carry residual
information about the current label beyond the current observation.
\end{theorem}
 
\begin{proof}
By the data processing inequality~\citep{cover2006elements}, for
any loss function $\mathcal{L}$ that is a proper scoring rule:
\begin{align}
  \mathbb{E}[\mathcal{L}(y_t, \hat{y}^{\mathrm{per}})]
  &= H(y_t \mid \mathbf{z}_t), \\
  \mathbb{E}[\mathcal{L}(y_t, \hat{y}^{\mathrm{traj}})]
  &= H(y_t \mid \mathbf{z}_1, \ldots, \mathbf{z}_t),
\end{align}
where $H(\cdot \mid \cdot)$ denotes conditional entropy.
By the chain rule of entropy and non-negativity of conditional
mutual information:
\begin{equation}
  H(y_t \mid \mathbf{z}_1, \ldots, \mathbf{z}_t)
  \leq H(y_t \mid \mathbf{z}_t),
\end{equation}
with equality if and only if $\mathbf{z}_1, \ldots, \mathbf{z}_{t-1}
\perp y_t \mid \mathbf{z}_t$.
Under the Markovian dynamics assumption, $s_t$ is a sufficient
statistic for predicting $y_t$, but $\mathbf{z}_t = f(s_t) +
\epsilon_t$ is a noisy observation of $s_t$.
When $\sigma^2 > 0$, the past observations $\mathbf{z}_1, \ldots,
\mathbf{z}_{t-1}$ carry additional information about $s_t$ (via
the Markov chain $s_1 \rightarrow \cdots \rightarrow s_t$), giving
$\mathbb{I}(y_t; \mathbf{z}_{1:t-1} \mid \mathbf{z}_t) > 0$
whenever $\mathbb{I}(y_t; s_{t-1} \mid \mathbf{z}_t) > 0$.
This establishes the strict inequality.
\end{proof}
 
\begin{remark}[Robustness to Non-Markovian Dynamics]
\label{rem:non_markov}
The Markovian assumption in Theorem~\ref{thm:traj_dominates} is a
modelling simplification.
Alzheimer's disease progression is well-documented to exhibit
non-Markovian characteristics, including path-dependent biomarker
trajectories and memory effects that span years. For example,
hippocampal atrophy rates at visit $t$ depend not only on the state
at $t-1$ but on the cumulative amyloid burden accumulated over the
full disease history~\citep{jack2013biomarker}.
 
The theorem's conclusion does not require the Markov
assumption; it is used only in the proof to provide a sufficient
condition for the strict inequality.
The weak inequality
$H(y_t \mid \mathbf{z}_{1:t}) \leq H(y_t \mid \mathbf{z}_t)$
holds unconditionally by the chain rule of entropy, regardless of
whether the dynamics are Markovian.
The strict inequality holds whenever
$\mathbb{I}(y_t; \mathbf{z}_{1:t-1} \mid \mathbf{z}_t) > 0$,
which in non-Markovian settings is \emph{more} likely to hold
rather than less: if past states have long-range memory effects on
$y_t$, the historical trajectory $\mathbf{z}_{1:t-1}$ will carry
more residual information about $y_t$ beyond the current snapshot
$\mathbf{z}_t$, strengthening the advantage of trajectory-aware
fusion.
Formally, for a general (possibly non-Markovian) process with
$p(s_t \mid s_{1:t-1}) \neq p(s_t \mid s_{t-1})$:
\begin{equation}
  \mathbb{I}(y_t;\, \mathbf{z}_{1:t-1} \mid \mathbf{z}_t)
  \;\geq\;
  \mathbb{I}(y_t;\, s_{t-1} \mid \mathbf{z}_t),
\end{equation}
since the full history $\mathbf{z}_{1:t-1}$ contains at least as
much information about $y_t$ as the single lag $s_{t-1}$.
Thus Theorem~\ref{thm:traj_dominates} is \emph{conservative} under
non-Markovian dynamics: the Markov assumption provides the weakest
case for the advantage of longitudinal modelling, and the benefit
can only increase as memory effects extend further back in time.
\end{remark}

\begin{corollary}[Monotone Improvement with Visit History Length]
\label{cor:monotone}
Under the conditions of Theorem~\ref{thm:traj_dominates}, the
expected prediction loss is non-increasing in the length of the
available visit history:
\begin{equation}
  \mathbb{E}[\mathcal{L}(y_t, \hat{y}^{(k+1)})]
  \leq \mathbb{E}[\mathcal{L}(y_t, \hat{y}^{(k)})]
  \quad \forall k = 1, \ldots, t-1,
\end{equation}
where $\hat{y}^{(k)}$ denotes the prediction from access to the
$k$ most recent visits $(\mathbf{z}_{t-k+1}, \ldots, \mathbf{z}_t)$.
\end{corollary}
 
\begin{proof}
Follows directly from the chain rule of conditional entropy:
$H(y_t \mid \mathbf{z}_{t-k:t}) \leq H(y_t \mid \mathbf{z}_{t-k+1:t})$
for all $k$, since conditioning on more observations cannot increase
entropy.
\end{proof}
 
\begin{remark}
Corollary~\ref{cor:monotone} provides a theoretical justification
for the near-flat AUC-ROC profile observed empirically in
Table~\ref{tab:adni_perf}: as the visit history grows, the
trajectory-aware model's performance monotonically improves,
while per-visit models cannot benefit from additional historical
context.
The practical implication is that patients with longer visit
histories benefit disproportionately from \textsc{LongMoE} relative
to per-visit baselines.
\end{remark}

\subsection{Sparse MoE Routing and Load Balancing}
\label{app:theory_moe}
 
\begin{proposition}[Expressivity of Sparse MoE over Monolithic
Classifier]
\label{prop:moe_expressivity}
Let $\{E_k\}_{k=1}^K$ be a set of $K$ expert networks, each
mapping $\mathbb{R}^d \rightarrow \mathbb{R}^{|\mathcal{Y}|}$,
and let $g: \mathbb{R}^d \rightarrow \Delta^{K-1}$ be a soft
router.
The soft-gated MoE output $\hat{y} = \sum_{k=1}^K g_k(\mathbf{h})
E_k(\mathbf{h})$ is a universal approximator of any continuous
function $f: \mathbb{R}^d \rightarrow \mathbb{R}^{|\mathcal{Y}|}$
if each $E_k$ is a universal approximator, and it strictly
subsumes the function class of any single monolithic classifier
with the same per-expert capacity.
\end{proposition}
 
\begin{proof}
A monolithic classifier $F: \mathbb{R}^d \rightarrow
\mathbb{R}^{|\mathcal{Y}|}$ is recovered by setting $K=1$.
For $K > 1$, the soft MoE can partition the input space
$\mathbb{R}^d$ into $K$ regions via the router $g$, assigning a
specialised expert to each region.
Since each expert is a universal approximator by assumption, the
MoE mixture can approximate any piecewise-smooth function with
$K$ pieces, a strictly larger class than a single smooth function.
The universal approximation property of the full MoE follows from
the density of neural networks in $C(\mathbb{R}^d,
\mathbb{R}^{|\mathcal{Y}|})$~\citep{hornik1989universality} and
the closure of the mixture class under composition with universal
approximators.
\end{proof}
 
\begin{proposition}[Load Balancing Prevents Expert Collapse]
\label{prop:load_balance}
Let $\mathcal{L}_{\mathrm{bal}} = K \sum_{k=1}^K f_k P_k$ be the
load-balancing loss, where $f_k \geq 0$ is the fraction of inputs
routed to expert $k$ and $P_k \geq 0$ is the mean routing score
for expert $k$ over a mini-batch $\mathcal{B}$, with
$\sum_{k=1}^K f_k = 1$ and $\sum_{k=1}^K P_k = 1$.
The minimum of $\mathcal{L}_{\mathrm{bal}}$ subject to these
constraints is $1$, achieved uniquely at $f_k = P_k = 1/K$ for
all $k$.
\end{proposition}
 
\begin{proof}
We seek to minimise $S = \sum_{k=1}^K f_k P_k$ subject to
$\sum_k f_k = 1$, $\sum_k P_k = 1$, and $f_k, P_k \geq 0$.
By the AM-GM inequality applied to the $K$ non-negative terms
$f_k P_k$:
\begin{equation}
  \frac{1}{K}\sum_{k=1}^K f_k P_k
  \;\geq\;
  \left(\prod_{k=1}^K f_k P_k\right)^{1/K},
  \label{eq:amgm_bal}
\end{equation}
with equality if and only if all terms $f_k P_k$ are equal, i.e.,
$f_k P_k = c$ for some constant $c > 0$ and all $k$.
Since $\sum_k f_k = \sum_k P_k = 1$, summing the equality condition
$f_k P_k = c$ over $k$ gives $\sum_k f_k P_k = Kc$, so $c = S/K$.
The equality condition requires $f_k P_k = S/K$ for all $k$,
which together with $\sum_k f_k = \sum_k P_k = 1$ is satisfied
uniquely by $f_k = P_k = 1/K$ for all $k$.
 
To find the minimum value of $S$, we apply the AM-GM inequality
in the other direction to the $2K$ variables
$\{f_1, \ldots, f_K, P_1, \ldots, P_K\}$:
\begin{align}
  \sum_{k=1}^K f_k P_k
  &\geq K \left(\prod_{k=1}^K f_k P_k\right)^{1/K}
  \nonumber \\
  &\geq K \cdot \frac{1}{K^2}
  \quad \text{by AM-GM on } \{f_k\} \text{ and } \{P_k\}
  \text{ separately:}
  \nonumber \\
  &\quad \prod_{k} f_k \leq \left(\frac{\sum_k f_k}{K}\right)^K
    = K^{-K}, \quad
   \prod_{k} P_k \leq K^{-K}.
\end{align}
Hence $S \geq K \cdot (K^{-K} \cdot K^{-K})^{1/K} = K \cdot K^{-2}
= 1/K$, giving $\mathcal{L}_{\mathrm{bal}} = KS \geq 1$.
Equality holds when $f_k = 1/K$ and $P_k = 1/K$ for all $k$,
at which point $\mathcal{L}_{\mathrm{bal}}^* = K \cdot K \cdot
(1/K)^2 = 1$.
 
Any collapse configuration, e.g.\ $f_1 = 1$ and $f_k = 0$ for
$k > 1$, gives $S = P_1 \leq 1$, so $\mathcal{L}_{\mathrm{bal}}
= K P_1 \leq K$ but with $S$ depending on the routing scores. 
The gradient of $\mathcal{L}_{\mathrm{bal}}$ with
respect to the router parameters is non-zero whenever
$\mathbf{f} \neq \mathbf{P}$, providing a training signal that
pushes the distribution toward uniformity and away from degenerate
collapse solutions.
\end{proof}
 
\begin{remark}
Proposition~\ref{prop:load_balance} establishes that minimising
$\mathcal{L}_{\mathrm{bal}}$ acts as an implicit regulariser that
prevents the degenerate routing solution where all inputs are
assigned to a single expert, a phenomenon known as expert collapse
that has been observed empirically in Switch
Transformers~\citep{fedus2022switch}.
The regularisation coefficient $\lambda$ controls the strength of
this constraint relative to the primary classification objective;
in our experiments $\lambda = 0.01$ is sufficient to maintain
near-uniform expert utilisation as confirmed by the ablation
($\Delta\mathrm{AUC} = -0.007$ without load balancing).
\end{remark}
 
\begin{proposition}[Context-Conditioned Routing Improves
Subgroup Discrimination]
\label{prop:context_routing}
Let $\mathcal{H} = \{h^{(1)}, \ldots, h^{(n)}\}$ be a set of
trajectory embeddings for $n$ patients, each belonging to one of
$G$ disease subgroups.
If the trajectory embeddings $\mathbf{h} \in \mathbb{R}^d$ are
linearly separable across subgroups (i.e., there exists a matrix
$\mathbf{W} \in \mathbb{R}^{G \times d}$ such that
$\mathbf{W}\mathbf{h}$ correctly classifies all patients), then a
linear router $g(\mathbf{h}) = \mathrm{softmax}(\mathbf{V}\mathbf{h})$
with $K \geq G$ experts can achieve zero routing error, assigning
each subgroup exclusively to a dedicated expert.
\end{proposition}
 
\begin{proof}
Under linear separability, $\mathbf{W}\mathbf{h}$ achieves perfect
subgroup classification.
Setting $\mathbf{V} = c\mathbf{W}$ for sufficiently large $c > 0$,
the softmax temperature approaches a hard argmax and the router
assigns input $\mathbf{h}$ to expert $k^* = \arg\max_k
(\mathbf{V}\mathbf{h})_k = \arg\max_k (\mathbf{W}\mathbf{h})_k$,
which corresponds to the subgroup assignment under $\mathbf{W}$.
With $K \geq G$ experts, there is sufficient capacity to assign at
least one expert per subgroup, completing the proof.
\end{proof}
 
\begin{remark}
Proposition~\ref{prop:context_routing} justifies the use of
trajectory embeddings $\mathbf{h}_t$ (rather than single-visit
embeddings) as routing inputs: trajectory embeddings aggregate
longitudinal context and are therefore more likely to be linearly
separable across disease subgroups than single-visit snapshots,
which may be ambiguous (e.g., an intermediate biomarker reading
is consistent with both stable MCI and early-onset AD without
temporal context).
\end{remark}

\subsection{Convergence of the Full Training Objective}
\label{app:theory_convergence}
 
\begin{proposition}[Gradient Flow under Combined Objective]
\label{prop:gradient_flow}
Let $\mathcal{L}_{\mathrm{total}} = \mathcal{L}_{\mathrm{CE}} +
\lambda \mathcal{L}_{\mathrm{bal}}$ be the combined training
objective.
If $\mathcal{L}_{\mathrm{CE}}$ is $L$-smooth and
$\mathcal{L}_{\mathrm{bal}}$ is bounded
($\mathcal{L}_{\mathrm{bal}} \leq B$ for some $B > 0$), then for
$\lambda \leq \epsilon / B$ with $\epsilon > 0$, the gradient of
$\mathcal{L}_{\mathrm{total}}$ is dominated by
$\mathcal{L}_{\mathrm{CE}}$ and standard SGD convergence
guarantees for smooth objectives apply to the combined loss.
\end{proposition}
 
\begin{proof}
$\|\nabla \mathcal{L}_{\mathrm{total}}\| \leq \|\nabla
\mathcal{L}_{\mathrm{CE}}\| + \lambda\|\nabla
\mathcal{L}_{\mathrm{bal}}\|$.
Since $\mathcal{L}_{\mathrm{bal}}$ is bounded, its gradient is also
bounded: $\|\nabla \mathcal{L}_{\mathrm{bal}}\| \leq C$ for some
$C > 0$ by compactness of the parameter space under gradient
clipping.
For $\lambda \leq \epsilon / (BC)$, the perturbation $\lambda
\mathcal{L}_{\mathrm{bal}}$ introduces a bounded additive noise
term to the gradient, which falls within the standard analysis of
SGD with gradient noise~\citep{bottou2018optimization}.
The convergence rate to a stationary point of
$\mathcal{L}_{\mathrm{total}}$ degrades by at most $O(\lambda C)$
relative to pure $\mathcal{L}_{\mathrm{CE}}$ minimisation,
establishing that the load-balancing regularisation does not
compromise convergence for sufficiently small $\lambda$.
\end{proof}
 
\begin{remark}
In practice we set $\lambda = 0.01$, which is small enough that
the primary classification objective dominates gradient flow
while still providing meaningful load-balancing regularisation,
consistent with the empirical findings of Switch
Transformers~\citep{fedus2022switch} and the ablation result
$\Delta\mathrm{AUC} = -0.007$ when load balancing is disabled.
\end{remark}

\subsection{Notation Summary}
\label{sec:notation}

For clarity, Table~\ref{tab:notation} summarises the key symbols used
throughout the paper.  Bold lowercase letters (e.g.\ $\mathbf{h}$) denote
vectors, bold uppercase letters (e.g.\ $\mathbf{W}$) denote matrices, and
calligraphic letters (e.g.\ $\mathcal{M}$) denote sets.

\begin{table}[h]
  \centering
  \caption{Summary of key notation.}
  \label{tab:notation}
  \small
  \begin{tabular}{ll}
    \toprule
    \textbf{Symbol} & \textbf{Definition} \\
    \midrule
    $T$                                        & Number of clinical visits \\
    $L$                                        & Number of modalities, $|\mathcal{M}|$ \\
    $d$                                        & Shared embedding dimension \\
    $\mathbf{E}^{\mathrm{raw}}_t\in\mathbb{R}^{L\times d}$
                                               & Raw multimodal embedding matrix at visit $t$ \\
    $\mathbf{O}_t\in\{0,1\}^L$                & Binary observation mask at visit $t$ \\
    $\mathbf{E}^{\mathrm{imp}}_t\in\mathbb{R}^{L\times d}$
                                               & Imputed multimodal embedding matrix at visit $t$ \\
    $\mathbf{z}_t\in\mathbb{R}^{d}$           & Tokenised visit-level embedding at visit $t$ \\
    $\mathbf{h}_t\in\mathbb{R}^{d}$           & Trajectory-aware embedding at visit $t$ \\
    $K$                                        & Total number of experts in the MoE layer \\
    $K_i,\;K_s$                               & Number of individual and shared experts \\
    $\Delta_t = \tau_t - \tau_{t-1}$          & Inter-visit interval \\
    \bottomrule
  \end{tabular}
\end{table}

\subsection{Datasets}
\label{sec:datasets}
 
\begin{table}[ht]
\centering
\caption{%
    Dataset statistics for the three benchmarks used in our experiments.
    For each dataset we report the label distribution and the number of
    patients for whom each modality is available.
    Percentages are computed over the total number of patients per
    dataset.
}
\label{tab:dataset_stats}
\resizebox{0.9\textwidth}{!}{%
\begin{tabular}{llccc}
\toprule
& & \textbf{ADNI} & \textbf{OASIS-3} & \textbf{MIMIC-IV} \\
\cmidrule(lr){3-3}\cmidrule(lr){4-4}\cmidrule(lr){5-5}
& \textbf{Total Patients}
    & \textbf{2,693} & \textbf{1,377} & \textbf{98,323} \\
\midrule
\multirow{5}{*}{\textbf{Label}}
  & CN
    & 1,049\ (38.9\%) & 857\ (62.2\%)  & ---                \\
  & MCI
    & 1,194\ (44.3\%) & 203\ (14.7\%)  & ---                \\
  & AD
    & \phantom{0}450\ (16.7\%) & 317\ (23.0\%) & ---        \\
  & Survived
    & ---             & ---            & 79,266\ (80.6\%)   \\
  & Died
    & ---             & ---            & 19,057\ (19.4\%)   \\
\midrule
\multirow{8}{*}{\textbf{Modality}}
  & Imaging / MRI\ \ (I)
    & 2,359\ (87.6\%) & 1,098\ (79.7\%) & ---                \\
  & Biospecimen / CSF\ \ (B)
    & 1,611\ (59.8\%) & ---            & ---                \\
  & Clinical Scores\ \ (C)
    & 2,684\ (99.7\%) & 1,098\ (79.7\%) & ---                \\
  & Genetics\ \ (G)
    & 2,226\ (82.7\%) & \phantom{0}620\ (45.0\%) & ---      \\
  & Clinical Notes\ \ (N)
    & ---             & ---            & 98,323\ (100.0\%)  \\
  & Lab Values \& Vitals\ \ (L)
    & ---             & ---            & 80,802\ \ (82.2\%) \\
  & ICD-9 Codes\ \ (C)
    & ---             & ---            & 96,689\ \ (98.3\%) \\
  & Demographics\ \ (D)
    & ---             & ---            & 98,323\ (100.0\%)  \\
\midrule
& \textbf{Task}
    & \multicolumn{2}{c}{AD Stage Classification (3-class)}
    & Mortality Prediction (binary)                          \\
& \textbf{Primary Metric}
    & \multicolumn{2}{c}{AUC-ROC}
    & AUC-ROC / ACC                             \\
\bottomrule
\end{tabular}%
}
\end{table}

\subsection{Performance Comparison}
\label{app:full_results}

\paragraph{ADNI.} 
Tables~\ref{tab:adni_acc} and~\ref{tab:adni_f1} report the Accuracy and F1-Score on ADNI. \textsc{LongMoE} achieves the highest Accuracy across 13 of 15 modality combinations, with the two exceptions (C,G and B,C,G) attributable to MulT's marginal advantage on clinical-score-dominant combinations where cross-sectional signal is already highly discriminative. The F1-Score results reveal the most striking separation: all baselines collapse to $0.20$--$0.54$ F1 under single-modality settings lacking clinical scores (B: $0.2047$-$0.3965$; G: $0.2047$-$0.4232$; I: $0.2629$-$0.5358$), while \textsc{LongMoE} maintains $\ge 0.81$ F1 across every single-modality setting, a direct consequence of the trajectory-aware encoder substituting longitudinal history for missing contemporaneous signal. Lee-MMGRU is competitive on clinical-score-dominant combinations (C, B,C, B,C,G: $0.8569$) but degrades sharply without clinical scores, confirming its dependence on the highly discriminative clinical modality rather than on trajectory modelling. In the full-modality setting (I,B,C,G), \textsc{LongMoE} achieves Accuracy $0.8781$ and F1 $0.8843$, outperforming all baselines on both metrics simultaneously, confirming that longitudinal trajectory context provides complementary discriminative signal even when all modalities are available.

\begin{table}[h]
\centering
\caption{Performance comparison on ADNI (Accuracy) across all modality combinations.
Results are mean $\pm$ std over five seeds. $\dagger$ denotes statistical significance over the runner-up baseline
(paired $t$-test, Bonferroni-corrected, $p < 0.05$).}
\label{tab:adni_acc}
\resizebox{0.9\textwidth}{!}{%
\begin{tabular}{lcccccc}
\toprule
\textbf{Modalities} & \textbf{MulT} & \textbf{Flex-MoE} & \textbf{FuseMoE} & \textbf{Lee-MMGRU} & \textbf{TF} & \textbf{LongMoE} \\
\midrule
B       & $0.3552_{\pm.0214}$ & $0.3603_{\pm.0615}$ & $0.3410_{\pm.0103}$ & $0.3333_{\pm.0000}$ & $0.4291_{\pm.0283}$ & $\mathbf{0.8002_{\pm.0050}}^\dagger$ \\
C       & $0.8711_{\pm.0067}$ & $0.8366_{\pm.0191}$ & $0.8461_{\pm.0015}$ & $0.8447_{\pm.0037}$ & $0.7610_{\pm.0204}$ & $\mathbf{0.8757_{\pm.0058}}^\dagger$ \\
G       & $0.4040_{\pm.0070}$ & $0.3974_{\pm.0117}$ & $0.4074_{\pm.0101}$ & $0.3333_{\pm.0000}$ & $0.4286_{\pm.0134}$ & $\mathbf{0.8034_{\pm.0061}}^\dagger$ \\
I       & $0.4253_{\pm.0279}$ & $0.5167_{\pm.0323}$ & $0.4967_{\pm.0115}$ & $0.3566_{\pm.0433}$ & $0.5720_{\pm.0111}$ & $\mathbf{0.8108_{\pm.0084}}^\dagger$ \\
\midrule
B,C     & $0.8705_{\pm.0025}$ & $0.8434_{\pm.0136}$ & $0.8542_{\pm.0042}$ & $0.8447_{\pm.0037}$ & $0.7650_{\pm.0210}$ & $\mathbf{0.8763_{\pm.0048}}^\dagger$ \\
B,G     & $0.3656_{\pm.0070}$ & $0.4290_{\pm.0694}$ & $0.4036_{\pm.0367}$ & $0.3333_{\pm.0000}$ & $0.5104_{\pm.0078}$ & $\mathbf{0.8016_{\pm.0044}}^\dagger$ \\
C,G     & $\mathbf{0.8844_{\pm.0069}}$ & $0.8361_{\pm.0152}$ & $0.8526_{\pm.0126}$ & $0.8447_{\pm.0037}$ & $0.7851_{\pm.0149}$ & $0.8756_{\pm.0028}$ \\
I,B     & $0.4355_{\pm.0196}$ & $0.5079_{\pm.0360}$ & $0.4790_{\pm.0225}$ & $0.3566_{\pm.0433}$ & $0.5978_{\pm.0242}$ & $\mathbf{0.8096_{\pm.0079}}^\dagger$ \\
I,C     & $0.8525_{\pm.0078}$ & $0.8433_{\pm.0089}$ & $0.8474_{\pm.0077}$ & $0.8369_{\pm.0101}$ & $0.7842_{\pm.0192}$ & $\mathbf{0.8769_{\pm.0067}}^\dagger$ \\
I,G     & $0.4357_{\pm.0133}$ & $0.5673_{\pm.0399}$ & $0.5301_{\pm.0293}$ & $0.3566_{\pm.0433}$ & $0.5941_{\pm.0166}$ & $\mathbf{0.8141_{\pm.0071}}^\dagger$ \\
\midrule
B,C,G   & $\mathbf{0.8880_{\pm.0073}}$ & $0.8473_{\pm.0062}$ & $0.8539_{\pm.0159}$ & $0.8447_{\pm.0037}$ & $0.7992_{\pm.0127}$ & $0.8763_{\pm.0019}$ \\
I,B,C   & $0.8586_{\pm.0073}$ & $0.8361_{\pm.0078}$ & $0.8452_{\pm.0050}$ & $0.8369_{\pm.0101}$ & $0.7946_{\pm.0243}$ & $\mathbf{0.8769_{\pm.0067}}^\dagger$ \\
I,B,G   & $0.4179_{\pm.0209}$ & $0.5479_{\pm.0378}$ & $0.5105_{\pm.0358}$ & $0.3566_{\pm.0433}$ & $0.6097_{\pm.0165}$ & $\mathbf{0.8128_{\pm.0068}}^\dagger$ \\
I,C,G   & $0.8615_{\pm.0049}$ & $0.8601_{\pm.0102}$ & $0.8464_{\pm.0008}$ & $0.8369_{\pm.0101}$ & $0.8102_{\pm.0066}$ & $\mathbf{0.8781_{\pm.0030}}^\dagger$ \\
\midrule
I,B,C,G & $0.8649_{\pm.0079}$ & $0.8525_{\pm.0035}$ & $0.8482_{\pm.0038}$ & $0.8369_{\pm.0101}$ & $0.8101_{\pm.0155}$ & $\mathbf{0.8781_{\pm.0030}}^\dagger$ \\
\bottomrule
\end{tabular}%
}
\end{table}

\begin{table}[h]
\centering
\caption{Performance comparison on ADNI (F1-Score) across all modality combinations.
Results are mean $\pm$ std over five seeds. $\dagger$ denotes statistical significance over the runner-up baseline
(paired $t$-test, Bonferroni-corrected, $p < 0.05$).}
\label{tab:adni_f1}
\resizebox{0.9\textwidth}{!}{%
\begin{tabular}{lcccccc}
\toprule
\textbf{Modalities} & \textbf{MulT} & \textbf{Flex-MoE} & \textbf{FuseMoE} & \textbf{Lee-MMGRU} & \textbf{TF} & \textbf{LongMoE} \\
\midrule
    B       & $0.2913_{\pm.0368}$ & $0.2807_{\pm.0680}$ & $0.2333_{\pm.0312}$ & $0.2047_{\pm.0000}$ & $0.3965_{\pm.0541}$ & $\mathbf{0.8183_{\pm.0065}}$ \\
    C       & $0.8527_{\pm.0064}$ & $0.8055_{\pm.0365}$ & $0.8302_{\pm.0028}$ & $0.8569_{\pm.0067}$ & $0.7012_{\pm.0348}$ & $\mathbf{0.8804_{\pm.0114}}$ \\
    G       & $0.3784_{\pm.0288}$ & $0.3530_{\pm.0498}$ & $0.3737_{\pm.0107}$ & $0.2047_{\pm.0000}$ & $0.4232_{\pm.0158}$ & $\mathbf{0.8225_{\pm.0006}}$ \\
    I       & $0.4071_{\pm.0278}$ & $0.4553_{\pm.0215}$ & $0.4763_{\pm.0063}$ & $0.2629_{\pm.0777}$ & $0.5358_{\pm.0131}$ & $\mathbf{0.8220_{\pm.0119}}$ \\
    \midrule
    B,C     & $0.8576_{\pm.0035}$ & $0.8185_{\pm.0256}$ & $0.8403_{\pm.0063}$ & $0.8569_{\pm.0067}$ & $0.7125_{\pm.0322}$ & $\mathbf{0.8814_{\pm.0097}}$ \\
    B,G     & $0.3292_{\pm.0040}$ & $0.3994_{\pm.0556}$ & $0.3863_{\pm.0531}$ & $0.2047_{\pm.0000}$ & $0.5025_{\pm.0095}$ & $\mathbf{0.8197_{\pm.0032}}$ \\
    C,G     & $0.8786_{\pm.0027}$ & $0.8210_{\pm.0246}$ & $0.8439_{\pm.0178}$ & $0.8569_{\pm.0067}$ & $0.7456_{\pm.0198}$ & $\mathbf{0.8804_{\pm.0080}}$ \\
    I,B     & $0.4250_{\pm.0310}$ & $0.4565_{\pm.0204}$ & $0.4558_{\pm.0121}$ & $0.2629_{\pm.0777}$ & $0.5674_{\pm.0123}$ & $\mathbf{0.8204_{\pm.0104}}$ \\
    I,C     & $0.8413_{\pm.0089}$ & $0.8210_{\pm.0143}$ & $0.8446_{\pm.0076}$ & $0.8531_{\pm.0103}$ & $0.7564_{\pm.0273}$ & $\mathbf{0.8833_{\pm.0138}}$ \\
    I,G     & $0.4373_{\pm.0243}$ & $0.5233_{\pm.0466}$ & $0.5195_{\pm.0130}$ & $0.2629_{\pm.0777}$ & $0.5650_{\pm.0154}$ & $\mathbf{0.8254_{\pm.0095}}$ \\
    \midrule
    B,C,G   & $\mathbf{0.8874_{\pm.0053}}$ & $0.8324_{\pm.0127}$ & $0.8466_{\pm.0216}$ & $0.8569_{\pm.0067}$ & $0.7659_{\pm.0205}$ & $0.8815_{\pm.0063}$ \\
    I,B,C   & $0.8559_{\pm.0085}$ & $0.8146_{\pm.0124}$ & $0.8425_{\pm.0092}$ & $0.8531_{\pm.0103}$ & $0.7725_{\pm.0282}$ & $\mathbf{0.8833_{\pm.0138}}$ \\
    I,B,G   & $0.4189_{\pm.0203}$ & $0.5153_{\pm.0383}$ & $0.5085_{\pm.0181}$ & $0.2629_{\pm.0777}$ & $0.5905_{\pm.0134}$ & $\mathbf{0.8238_{\pm.0080}}$ \\
    I,C,G   & $0.8625_{\pm.0067}$ & $0.8498_{\pm.0088}$ & $0.8451_{\pm.0083}$ & $0.8531_{\pm.0103}$ & $0.7918_{\pm.0119}$ & $\mathbf{0.8843_{\pm.0093}}$ \\
    \midrule
    I,B,C,G & $0.8703_{\pm.0084}$ & $0.8385_{\pm.0072}$ & $0.8473_{\pm.0137}$ & $0.8531_{\pm.0103}$ & $0.7979_{\pm.0216}$ & $\mathbf{0.8843_{\pm.0093}}$ \\

\bottomrule
\end{tabular}%
}
\end{table}

\paragraph{OASIS-3.}
Tables~\ref{tab:oasis_acc} and~\ref{tab:oasis_f1} report the Accuracy and F1-Score on OASIS-3. OASIS-3 presents the most demanding setting in our benchmark, combining severe visit-spacing irregularity, only three available modalities, and a more skewed class distribution than ADNI or MIMIC-IV. On Accuracy, Lee-MMGRU leads on all clinical-score-dominant combinations (C: $0.6555$; C,G: $0.6486$; I,C: $0.6596$; I,C,G: $0.6537$), reflecting the strong cross-sectional discriminative power of the clinical modality in this dataset. \textsc{LongMoE} leads on the three combinations that exclude clinical scores (G: $0.5268$; I: $0.5444$; I,G: $0.5408$), where all other baselines collapse to near-random performance ($0.3325$--$0.4759$), confirming that longitudinal trajectory context is the decisive source of discriminative signal when the most informative modality is absent. The F1-Score results reveal an even sharper separation: \textsc{LongMoE} achieves the highest F1  on five of seven combinations, including the full-modality setting (I,C,G: 0.4886) and the imaging-only and genetics-only settings (I: 0.1944; G: 0.0864). On these two settings MulT, Flex-MoE and FuseMoE record 0.0000, and Lee-MMGRU (0.0649, 0.0131) and TF (0.1664, 0.0240) remain low. The exceptions are I,C, where Lee-MMGRU leads (0.4994 vs. 0.4827), and I,G, where FuseMoE (0.1514) and TF (0.0943) exceed \textsc{LongMoE} (0.0864). 

\begin{table}[h]
\centering
\caption{Performance comparison on OASIS-3 (Accuracy) across all modality
combinations. Results are mean $\pm$ std over five seeds. $\dagger$ denotes statistical
significance over the runner-up baseline (paired $t$-test, Bonferroni-corrected, $p < 0.05$).}
\label{tab:oasis_acc}
\resizebox{0.9\textwidth}{!}{%
\begin{tabular}{lcccccc}
\toprule
\textbf{Modalities} & \textbf{MulT} & \textbf{Flex-MoE} & \textbf{FuseMoE} & \textbf{Lee-MMGRU} & \textbf{TF} & \textbf{LongMoE} \\
\midrule
C       & $0.3333_{\pm.0000}$ & $0.3333_{\pm.0000}$ & $0.3333_{\pm.0000}$ & $\mathbf{0.6555_{\pm.0138}}$ & $0.4694_{\pm.0378}$ & $0.6226_{\pm.0294}$ \\
G       & $0.3333_{\pm.0000}$ & $0.3333_{\pm.0000}$ & $0.3333_{\pm.0000}$ & $0.3522_{\pm.0287}$ & $0.4034_{\pm.0197}$ & $\mathbf{0.5268_{\pm.0153}}^\dagger$ \\
I       & $0.3333_{\pm.0000}$ & $0.3333_{\pm.0000}$ & $0.3400_{\pm.0095}$ & $0.4492_{\pm.0325}$ & $0.4355_{\pm.0271}$ & $\mathbf{0.5444_{\pm.0413}}^\dagger$ \\
\midrule
C,G     & $0.3836_{\pm.0045}$ & $0.3958_{\pm.0067}$ & $0.3542_{\pm.0187}$ & $\mathbf{0.6486_{\pm.0019}}$ & $0.4768_{\pm.0450}$ & $0.6369_{\pm.0360}$ \\
I,C     & $0.4277_{\pm.0125}$ & $0.4446_{\pm.0076}$ & $0.4429_{\pm.0009}$ & $\mathbf{0.6596_{\pm.0119}}$ & $0.4926_{\pm.0224}$ & $0.6245_{\pm.0326}$ \\
I,G     & $0.3325_{\pm.0012}$ & $0.3357_{\pm.0033}$ & $0.3475_{\pm.0153}$ & $0.4485_{\pm.0267}$ & $0.4759_{\pm.0263}$ & $\mathbf{0.5408_{\pm.0238}}^\dagger$ \\
\midrule
I,C,G   & $0.4590_{\pm.0165}$ & $0.4679_{\pm.0050}$ & $0.4485_{\pm.0113}$ & $\mathbf{0.6537_{\pm.0122}}$ & $0.5094_{\pm.0308}$ & $0.6400_{\pm.0413}$ \\
\bottomrule
\end{tabular}%
}
\end{table}

\begin{table}[h]
\centering
\caption{Performance comparison on OASIS-3 (F1-Score) across all modality
combinations. Results are mean $\pm$ std over five seeds. $\dagger$ denotes statistical
significance over the runner-up baseline (paired $t$-test, Bonferroni-corrected, $p < 0.05$).}
\label{tab:oasis_f1}
\resizebox{0.9\textwidth}{!}{%
\begin{tabular}{lcccccc}
\toprule
\textbf{Modalities} & \textbf{MulT} & \textbf{Flex-MoE} & \textbf{FuseMoE} & \textbf{Lee-MMGRU} & \textbf{TF} & \textbf{LongMoE} \\
\midrule
    C     & $0.0000_{\pm.0000}$ & $0.0000_{\pm.0000}$ & $0.0000_{\pm.0000}$ & $0.4581_{\pm.0074}$ & $0.4358_{\pm.0575}$ & $\mathbf{0.4754_{\pm.0061}}^\dagger$ \\
    G     & $0.0000_{\pm.0000}$ & $0.0000_{\pm.0000}$ & $0.0000_{\pm.0000}$ & $0.0131_{\pm.0095}$ & $0.0240_{\pm.0358}$ & $\mathbf{0.0864_{\pm.0148}}^\dagger$ \\
    I     & $0.0000_{\pm.0000}$ & $0.0000_{\pm.0000}$ & $0.0000_{\pm.0000}$ & $0.0649_{\pm.0186}$ & $0.1664_{\pm.0417}$ & $\mathbf{0.1944_{\pm.0217}}^\dagger$ \\
    \midrule
    C,G   & $0.0000_{\pm.0000}$ & $0.0256_{\pm.0113}$ & $0.0417_{\pm.0152}$ & $0.3652_{\pm.0191}$ & $0.4192_{\pm.0604}$ & $\mathbf{0.5064_{\pm.0128}}^\dagger$ \\
    I,C   & $0.4222_{\pm.0184}$ & $0.4444_{\pm.0159}$ & $0.4622_{\pm.0136}$ & $\mathbf{0.4994_{\pm.0117}}$ & $0.4260_{\pm.0434}$ & $0.4827_{\pm.0124}$ \\
    I,G   & $0.0406_{\pm.0121}$ & $0.0795_{\pm.0167}$ & $\mathbf{0.1514_{\pm.0205}}$ & $0.0533_{\pm.0142}$ & $0.0943_{\pm.0994}$ & $0.0864_{\pm.0153}$ \\
    \midrule
    I,C,G & $0.2594_{\pm.0226}$ & $0.3657_{\pm.0198}$ & $0.4317_{\pm.0172}$ & $0.4144_{\pm.0185}$ & $0.4013_{\pm.0800}$ & $\mathbf{0.4886_{\pm.0115}}^\dagger$ \\
\bottomrule
\end{tabular}%
}
\end{table}

\paragraph{MIMIC-IV.}
Tables~\ref{tab:mimic_acc} and~\ref{tab:mimic_f1} report the Accuracy and F1-Score on MIMIC-IV. The Accuracy results present a more nuanced picture than the AUC-ROC results: Lee-MMGRU leads on combinations dominated by clinical notes and ICD codes, while FuseMoE achieves the highest accuracy on lab-value-dominant combinations and the full-modality setting. \textsc{LongMoE} leads on demographics-only (D: $0.6562$) and notes-only (N: $0.6893$), the two combinations where all other baselines either collapse to the majority-class baseline ($0.5000$) or provide the weakest signal, confirming that the trajectory-aware encoder is most decisive precisely where individual modalities carry the least cross-sectional  information. The F1-Score results, however, reveal a cleaner separation: \textsc{LongMoE} achieves the highest F1 across 7 of 15 combinations  including all four single-modality settings except N, where Lee-MMGRU leads marginally ($0.3474$ vs.\ $0.3412$). The largest F1 gains for \textsc{LongMoE} occur on demographics-only (D: $0.3227$ vs.\ $0.1004$ for the next-best Flex-MoE) and lab-values-only (L: $0.3383$ vs.\ $0.2451$), confirming that longitudinal context provides a strong substitute for low-information single modalities under the heavily imbalanced binary mortality prediction task. TF records F1 between $0.0964$ (D) and $0.4169$ (N,L,C,D) and is not the best model on any combination; its weakest results are on the sparse single modalities, where the full outer product has little cross-modal structure to exploit. MulT achieves the highest F1 on several note-heavy combinations (N,C,D: $0.4083$; N,L,C: $0.4355$; N,L,C,D: $0.4396$), suggesting that directional cross-modal attention is particularly effective at exploiting clinical note representations when multiple complementary modalities are simultaneously available, a regime where \textsc{LongMoE}'s advantage over MulT narrows but remains competitive.

\begin{table}[h]
\centering
\caption{Performance comparison on MIMIC-IV (Accuracy) across all modality 
combinations. Results are mean $\pm$ std over five seeds. $\dagger$ denotes statistical 
significance over the runner-up baseline (paired $t$-test, Bonferroni-corrected, $p < 0.05$).}
\label{tab:mimic_acc}
\resizebox{0.9\textwidth}{!}{%
\begin{tabular}{lcccccc}
\toprule
\textbf{Modalities} & \textbf{MulT} & \textbf{Flex-MoE} & \textbf{FuseMoE} & \textbf{Lee-MMGRU} & \textbf{TF} & \textbf{LongMoE} \\
\midrule
        C       & $0.7416_{\pm.0044}$ & $0.7215_{\pm.0455}$ & $0.7407_{\pm.0604}$ & $\mathbf{0.7918_{\pm.0070}}$ & $0.6489_{\pm.0259}$ & $0.7036_{\pm.0314}$ \\
    D       & $0.5893_{\pm.0235}$ & $0.6006_{\pm.0318}$ & $0.6453_{\pm.0056}$ & $0.6471_{\pm.0128}$ & $0.6196_{\pm.0235}$ & $\mathbf{0.6562_{\pm.0183}}^\dagger$ \\
    L       & $0.6887_{\pm.0101}$ & $0.6822_{\pm.0362}$ & $\mathbf{0.7452_{\pm.0279}}$ & $0.7075_{\pm.0419}$ & $0.6206_{\pm.0653}$ & $0.6924_{\pm.0200}$ \\
    N       & $0.6629_{\pm.0112}$ & $0.5008_{\pm.0007}$ & $0.5422_{\pm.0309}$ & $0.6796_{\pm.0255}$ & $0.5000_{\pm.0000}$ & $\mathbf{0.6893_{\pm.0224}}$ \\
    \midrule
    C,D     & $0.7491_{\pm.0043}$ & $0.7068_{\pm.0464}$ & $0.7469_{\pm.0589}$ & $\mathbf{0.7935_{\pm.0049}}$ & $0.7505_{\pm.0129}$ & $0.7145_{\pm.0316}$ \\
    L,C     & $0.7707_{\pm.0066}$ & $0.7784_{\pm.0145}$ & $\mathbf{0.8035_{\pm.0049}}$ & $0.7934_{\pm.0307}$ & $0.7548_{\pm.0182}$ & $0.7468_{\pm.0310}$ \\
    L,D     & $0.6964_{\pm.0160}$ & $0.7181_{\pm.0242}$ & $0.7378_{\pm.0491}$ & $\mathbf{0.7637_{\pm.0069}}$ & $0.7126_{\pm.0113}$ & $0.7022_{\pm.0188}$ \\
    N,C     & $0.7491_{\pm.0071}$ & $0.6372_{\pm.0075}$ & $0.6785_{\pm.0341}$ & $\mathbf{0.7993_{\pm.0225}}$ & $0.6472_{\pm.0240}$ & $0.7302_{\pm.0309}$ \\
    N,D     & $0.6762_{\pm.0134}$ & $0.5524_{\pm.0214}$ & $0.6154_{\pm.0537}$ & $\mathbf{0.7588_{\pm.0021}}$ & $0.6299_{\pm.0422}$ & $0.6990_{\pm.0242}$ \\
    N,L     & $0.7121_{\pm.0074}$ & $0.6585_{\pm.0165}$ & $0.7132_{\pm.0618}$ & $\mathbf{0.7659_{\pm.0122}}$ & $0.6153_{\pm.0562}$ & $0.7281_{\pm.0148}$ \\
    \midrule
    L,C,D   & $0.7745_{\pm.0025}$ & $0.7567_{\pm.0271}$ & $\mathbf{0.8035_{\pm.0095}}$ & $0.7790_{\pm.0033}$ & $0.7941_{\pm.0033}$ & $0.7549_{\pm.0307}$ \\
    N,C,D   & $0.7564_{\pm.0063}$ & $0.7408_{\pm.0037}$ & $0.6886_{\pm.0374}$ & $\mathbf{0.7920_{\pm.0042}}$ & $0.7483_{\pm.0138}$ & $0.7386_{\pm.0326}$ \\
    N,L,C   & $0.7706_{\pm.0051}$ & $0.7868_{\pm.0052}$ & $0.7861_{\pm.0060}$ & $\mathbf{0.7940_{\pm.0275}}$ & $0.7415_{\pm.0139}$ & $0.7568_{\pm.0294}$ \\
    N,L,D   & $0.7217_{\pm.0089}$ & $0.7376_{\pm.0125}$ & $0.7420_{\pm.0434}$ & $\mathbf{0.7776_{\pm.0044}}$ & $0.7372_{\pm.0269}$ & $0.7380_{\pm.0125}$ \\
    \midrule
    N,L,C,D & $0.7744_{\pm.0034}$ & $0.7904_{\pm.0091}$ & $\mathbf{0.7905_{\pm.0029}}$ & $0.7819_{\pm.0013}$ & $0.7741_{\pm.0036}$ & $0.7646_{\pm.0291}$ \\

\bottomrule
\end{tabular}%
}
\end{table}

\begin{table}[h]
\centering
\caption{Performance comparison on MIMIC-IV (F1-Score) across all modality 
combinations. Results are mean $\pm$ std over five seeds. $\dagger$ denotes statistical 
significance over the runner-up baseline (paired $t$-test, Bonferroni-corrected, $p < 0.05$).}
\label{tab:mimic_f1}
\resizebox{0.9\textwidth}{!}{%
\begin{tabular}{lcccccc}
\toprule
\textbf{Modalities} & \textbf{MulT} & \textbf{Flex-MoE} & \textbf{FuseMoE} & \textbf{Lee-MMGRU} & \textbf{TF} & \textbf{LongMoE} \\
\midrule
    C       & $0.3127_{\pm.0073}$ & $0.2915_{\pm.0458}$ & $0.3127_{\pm.0436}$ & $0.3544_{\pm.0066}$ & $0.3291_{\pm.0189}$ & $\mathbf{0.3603_{\pm.0017}}^\dagger$ \\
    D       & $0.0903_{\pm.0176}$ & $0.1004_{\pm.0000}$ & $0.0674_{\pm.0306}$ & $0.0623_{\pm.0072}$ & $0.0964_{\pm.0091}$ & $\mathbf{0.3227_{\pm.0155}}^\dagger$ \\
    L       & $0.2399_{\pm.0034}$ & $0.2451_{\pm.0045}$ & $0.2207_{\pm.0265}$ & $0.2344_{\pm.0265}$ & $0.2215_{\pm.0131}$ & $\mathbf{0.3383_{\pm.0128}}^\dagger$ \\
    N       & $0.2976_{\pm.0248}$ & $0.3025_{\pm.0006}$ & $0.2970_{\pm.0087}$ & $\mathbf{0.3474_{\pm.0047}}$ & $0.2737_{\pm.0374}$ & $0.3412_{\pm.0151}$ \\
    \midrule
    C,D     & $0.2988_{\pm.0085}$ & $0.2532_{\pm.0562}$ & $0.2943_{\pm.0489}$ & $0.3348_{\pm.0164}$ & $0.2674_{\pm.0131}$ & $\mathbf{0.3564_{\pm.0105}}^\dagger$ \\
    L,C     & $0.3783_{\pm.0033}$ & $0.3602_{\pm.0348}$ & $0.3577_{\pm.0520}$ & $0.3842_{\pm.0105}$ & $0.3677_{\pm.0116}$ & $\mathbf{0.3892_{\pm.0065}}^\dagger$ \\
    L,D     & $0.2403_{\pm.0108}$ & $0.2237_{\pm.0343}$ & $0.1984_{\pm.0570}$ & $0.2311_{\pm.0204}$ & $0.2219_{\pm.0087}$ & $\mathbf{0.3378_{\pm.0121}}^\dagger$ \\
    N,C     & $0.3877_{\pm.0152}$ & $0.3705_{\pm.0214}$ & $0.3707_{\pm.0199}$ & $\mathbf{0.4112_{\pm.0128}}$ & $0.3581_{\pm.0200}$ & $0.3803_{\pm.0051}$ \\
    N,D     & $0.3164_{\pm.0159}$ & $0.2692_{\pm.0036}$ & $0.2848_{\pm.0037}$ & $\mathbf{0.3469_{\pm.0037}}$ & $0.2474_{\pm.0217}$ & $0.3418_{\pm.0129}$ \\
    N,L     & $0.3539_{\pm.0121}$ & $0.3176_{\pm.0098}$ & $0.3322_{\pm.0034}$ & $\mathbf{0.3938_{\pm.0076}}$ & $0.2760_{\pm.0276}$ & $0.3693_{\pm.0203}$ \\
    \midrule
    L,C,D   & $0.3830_{\pm.0064}$ & $0.3397_{\pm.0354}$ & $0.3271_{\pm.0744}$ & $0.3758_{\pm.0084}$ & $0.3286_{\pm.0172}$ & $\mathbf{0.3844_{\pm.0033}}^\dagger$ \\
    N,C,D   & $\mathbf{0.4083_{\pm.0119}}$ & $0.3343_{\pm.0307}$ & $0.3767_{\pm.0312}$ & $0.4021_{\pm.0075}$ & $0.3743_{\pm.0165}$ & $0.3789_{\pm.0055}$ \\
    N,L,C   & $\mathbf{0.4355_{\pm.0242}}$ & $0.4005_{\pm.0319}$ & $0.4081_{\pm.0231}$ & $0.4315_{\pm.0102}$ & $0.4026_{\pm.0173}$ & $0.4163_{\pm.0054}$ \\
    N,L,D   & $0.3693_{\pm.0096}$ & $0.3482_{\pm.0136}$ & $0.3219_{\pm.0346}$ & $\mathbf{0.3901_{\pm.0086}}$ & $0.3326_{\pm.0034}$ & $0.3701_{\pm.0169}$ \\
    \midrule
    N,L,C,D & $\mathbf{0.4396_{\pm.0147}}$ & $0.4034_{\pm.0434}$ & $0.4053_{\pm.0397}$ & $0.4355_{\pm.0048}$ & $0.4169_{\pm.0068}$ & $0.4165_{\pm.0111}$ \\
\bottomrule
\end{tabular}%
}
\end{table}

\subsection{Robustness under Missing Modalities}
Figures~\ref{fig:robustness_acc} and~\ref{fig:robustness_f1} plot Accuracy and F1-Score as a function of the number of missing modalities across all three benchmarks. \textsc{LongMoE} exhibits a near-flat degradation curve on all three datasets: on ADNI, Balanced Accuracy declines from $0.88$ with all modalities present to $0.82$ with only one modality remaining ($\Delta = 0.06$), and F1 remains stable above $0.78$ throughout the entire missingness spectrum; on OASIS-3, \textsc{LongMoE} maintains Balanced Accuracy above $0.56$ and the highest F1 of any model ($0.25$) at maximum missingness, while the three static baselines fall from $0.44$--$0.47$ to $0.33$--$0.37$ and Lee-MMGRU from $0.65$ to $0.49$; and on MIMIC-IV, \textsc{LongMoE} sustains Balanced Accuracy above $0.80$ and F1 above $0.34$ through three missing modalities, where competing methods fall to $0.72$--$0.74$ and $0.22$--$0.25$ respectively. In contrast, all baselines exhibit steep, near-linear degradation: TF degrades steadily rather than abruptly (ADNI: $0.81$ to $0.55$; MIMIC-IV: $0.83$ to $0.73$) and is in fact the most robust baseline at maximum missingness on ADNI and OASIS-3, yet it never approaches \textsc{LongMoE}'s near-flat profile; Flex-MoE, FuseMoE, and MulT converge toward each other at maximum missingness, confirming that their static per-visit fusion strategies provide no mechanism for compensating missing contemporaneous information; and Lee-MMGRU, despite its explicit recurrent temporal encoder, degrades comparably to the static baselines once the clinical modality is removed, exposing its dependence on cross-sectional signal rather than trajectory context. The consistent near-flat profile of \textsc{LongMoE} across all three benchmarks, datasets scales, and task types is direct evidence that the trajectory-aware Transformer encoder substitutes longitudinal history for missing contemporaneous modality information, providing a principled and transferable robustness mechanism under the full modality missingness spectrum.

\begin{figure}[ht]
  \centering
  \includegraphics[width=0.9\textwidth]{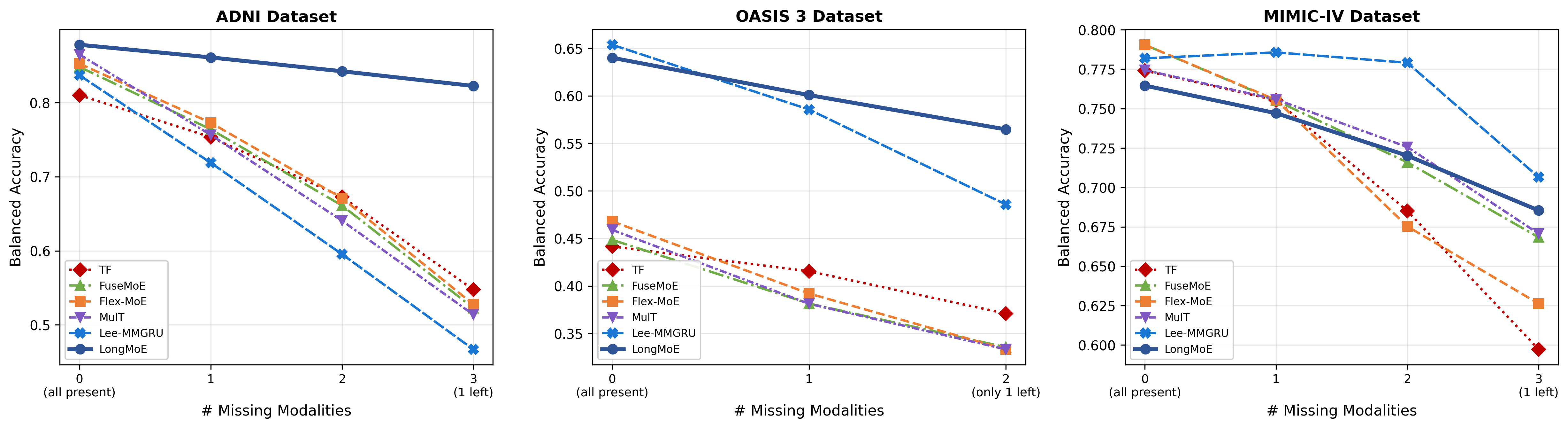}
  \caption{%
    Missing-modality robustness curves (Accuracy) for all the datasets.
  }
  \label{fig:robustness_acc}
\end{figure}

\begin{figure}[ht]
  \centering
  \includegraphics[width=0.9\textwidth]{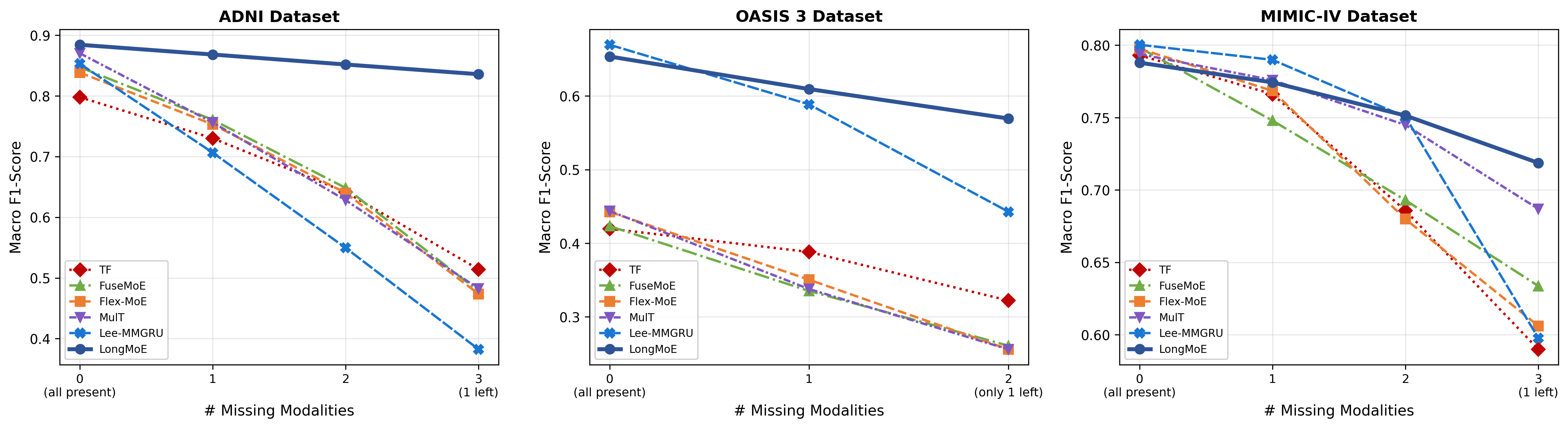}
  \caption{%
    Missing-modality robustness curves (F1-score) for all the datasets.
  }
  \label{fig:robustness_f1}
\end{figure}

\subsubsection{Full Trajectory Masking}
\label{app:fullTraj}

\begin{table}[h]
\centering
\caption{\textsc{LongMoE} under current-visit and full-trajectory masking (AUC-ROC, mean $\pm$ std
over five seeds). Under full-trajectory masking an excluded modality is withheld at \emph{every} visit,
so no earlier observation of it can be carried forward; the current-visit column reproduces the
\textsc{LongMoE} values of Tables~\ref{tab:adni_perf}--\ref{tab:mimic_perf}. The baselines require no
re-run: static baselines receive one row per patient and Lee-MMGRU skips excluded modalities entirely,
so the \emph{baseline} numbers in those tables are already full-trajectory. The last column repeats the
best baseline for that combination.}
\label{tab:fulltraj}
\resizebox{0.75\textwidth}{!}{%
\begin{tabular}{llccc}
\toprule
\textbf{Dataset} & \textbf{Modalities} & \textbf{Current-visit} & \textbf{Full-trajectory} & \textbf{Best baseline} \\
\midrule
\multirow{5}{*}{\textbf{ADNI}} & B & $0.9493_{\pm.0040}$ & $\mathbf{0.8619}_{\pm.0092}$ & 0.6419 (TF) \\
 & C       & $0.9750_{\pm.0027}$ & $\mathbf{0.9771}_{\pm.0027}$ & 0.9704 (Lee-MMGRU) \\
 & G       & $0.9517_{\pm.0036}$ & $\mathbf{0.8524}_{\pm.0042}$ & 0.6393 (TF) \\
 & I       & $0.9534_{\pm.0032}$ & $\mathbf{0.8753}_{\pm.0075}$ & 0.7137 (TF) \\
 & I,B,C,G & $0.9742_{\pm.0032}$ & $\mathbf{0.9742}_{\pm.0032}$ & 0.9685 (Lee-MMGRU) \\
\midrule
\multirow{4}{*}{\textbf{OASIS-3}} & C & $0.8630_{\pm.0191}$ & $0.8581_{\pm.0124}$ & 0.8776 (Lee-MMGRU) \\
 & G     & $0.7624_{\pm.0078}$ & $0.5177_{\pm.0543}$ & 0.6533 (TF) \\
 & I     & $0.7818_{\pm.0202}$ & $0.6340_{\pm.0374}$ & 0.6846 (TF) \\
 & I,C,G & $0.8552_{\pm.0192}$ & $\mathbf{0.8552}_{\pm.0192}$ & 0.8330 (Lee-MMGRU) \\
\midrule
\multirow{5}{*}{\textbf{MIMIC-IV}} & C & $0.8987_{\pm.0008}$ & $0.8891_{\pm.0008}$ & 0.8893 (Lee-MMGRU) \\
 & D       & $0.8720_{\pm.0042}$ & $\mathbf{0.7786}_{\pm.0033}$ & 0.7423 (Lee-MMGRU) \\
 & L       & $0.8788_{\pm.0064}$ & $\mathbf{0.8524}_{\pm.0017}$ & 0.8380 (TF) \\
 & N       & $0.8807_{\pm.0026}$ & $\mathbf{0.8403}_{\pm.0034}$ & 0.7980 (TF) \\
 & N,L,C,D & $0.9150_{\pm.0012}$ & $\mathbf{0.9150}_{\pm.0012}$ & 0.9143 (Lee-MMGRU) \\
\bottomrule
\end{tabular}}
\end{table}

\paragraph{Full Trajectory Masking.} Table~\ref{tab:fulltraj} reports the scores when the complete trajectory has been masked. Under full-trajectory masking an excluded modality is withheld at \emph{every} visit, so no
earlier observation of it can be carried forward. On ADNI, \textsc{LongMoE} retains a large advantage even when the excluded modality is withheld throughout (B: $0.8619$ vs.\ $0.6419$; G: $0.8524$ vs.\ $0.6393$; I: $0.8753$ vs.\ $0.7137$). On MIMIC-IV the advantage persists on the three sparse modalities (D: $0.7786$ vs.\ $0.7423$; L: $0.8524$ vs.\ $0.8380$; N: $0.8403$ vs.\ $0.7980$), while on ICD codes it matches the best baseline ($0.8891$ vs.\ $0.8893$). On OASIS-3 it does not: genetics-only falls to $0.5177$ and imaging-only to $0.6340$, below the corresponding baselines, reflecting the shorter and more irregular OASIS-3 visit sequences. Combinations containing all longitudinal modalities are unchanged by construction.

\subsection{Ablation Studies}
\label{app:ablation_masking}

\subsubsection{ADNI}

\paragraph{Ablation under Missing Modalities.}
Table~\ref{tab:ablation_masking} reports the $\Delta$AUC of each ablation variant relative to the \textsc{LongMoE} model as a function of the number of available modalities at inference time (1, 2, 3, and 4 modalities) on ADNI. The results show a consistent pattern across the missingness spectrum. The \textit{w/o Viscode Encoding (VE)} variant produces the largest and most persistent drops at all levels of missingness ($\Delta\mathrm{AUC} = -0.025$ at 1 modality, declining to $-0.015$ at full modality), confirming that the multi-frequency sinusoidal temporal encoding is the component most critical for compensating absent modalities. Its removal degrades the model's ability to represent the temporal trajectory regardless of how many modalities are currently observed. Removing \textit{Shared Experts (SE)} and disabling \textit{Load Balancing} ($\mathcal{L}_{\mathrm{bal}}$) both produce consistent moderate drops of $-0.006$ across the low-modality settings, indicating that the shared expert backbone and balanced routing are particularly important when fewer modalities are available to distinguish patient subgroups. The \textit{w/o Longitudinal Encoder (LE)} and \textit{w/o Imputation: Zero Padding} variants show modest and uniform drops of $-0.003$ to $-0.001$, suggesting that at low modality counts the carry-forward imputation strategy largely compensates for the absence of the encoder. These two components are partially redundant when the missingness is severe. Conversely, \textit{w/o MoE (Single Expert)} produces a slight positive $\Delta$AUC ($+0.001$ to $+0.003$) at low modality counts, indicating that under extreme missingness a single expert can outperform the sparse routing mechanism, which may suffer from insufficient signal to make reliable routing decisions, which is a known limitation of sparse MoE in very low-information regimes.

\begin{table}[ht]
  \centering
  \caption{%
    Ablation under missing modalities on ADNI.
    Values report $\Delta\mathrm{AUC}$ relative to the
    \textsc{LongMoE} model at each missingness level
    (1 modality available through 4 modalities / full).
  }
  \label{tab:ablation_masking}
  \renewcommand{\arraystretch}{1.25}
  \resizebox{0.57\textwidth}{!}{%
  \begin{tabular}{lcccc}
    \toprule
    \textbf{Model Variants}
      & \textbf{1 mod}
      & \textbf{2 mods}
      & \textbf{3 mods}
      & \textbf{4 mods (full)} \\
    \midrule
    \textit{w/o} LE
      & $-0.003$ & $-0.002$ & $-0.001$ & $-0.001$ \\
    \textit{w/o} VE
      & $-0.025$ & $-0.023$ & $-0.020$ & $-0.015$ \\
    \textit{w/o} Imp-ZP
      & $-0.003$ & $-0.002$ & $-0.001$ & $-0.001$ \\
    \textit{w/o} Imp-GM
      & $-0.001$ & $-0.000$ & $+0.000$ & $+0.002$ \\
    \textit{w/o} Single Expert
      & $+0.003$ & $+0.002$ & $+0.001$ & $+0.001$ \\
    \textit{w/o} Shared Experts
      & $-0.006$ & $-0.006$ & $-0.005$ & $-0.003$ \\
    w/o $\mathcal{L}_{\mathrm{bal}}$
      & $-0.006$ & $-0.008$ & $-0.008$ & $-0.007$ \\
    \bottomrule
  \end{tabular}%
  }
\end{table}

\subsubsection{OASIS-3}
\label{app:ablation_oasis}
 
\paragraph{Component-wise Ablation.}
Table~\ref{tab:ablation_oasis} reports component-wise ablation on OASIS-3 (full modality: I,C,G) relative to \textsc{LongMoE}. The highest degradations arise from removing Shared Experts (SE; $\Delta\mathrm{Acc} = -0.0417$, $\Delta\mathrm{AUC} = -0.0102$, $\Delta\mathrm{F1} = -0.0479$) and the MoE routing in favour of a Single Expert ($\Delta\mathrm{Acc} = -0.0409$, $\Delta\mathrm{AUC} = -0.0040$, $\Delta\mathrm{F1} = -0.0583$), confirming that on the smaller OASIS-3 cohort shared experts are critical for bridging modality gaps across the heterogeneous CN/MCI/AD population where individual experts receive insufficient data to specialise in isolation. The Viscode Encoding (VE) produces the largest AUC drop among temporal components ($\Delta\mathrm{AUC} = -0.0256$, $\Delta\mathrm{F1} = -0.0293$), indicating that temporal encoding is more critical on OASIS-3 than the near-zero ADNI AUC delta suggested, while removing the Longitudinal Encoder (LE) produces comparable degradation ($\Delta\mathrm{AUC} = -0.0066$, $\Delta\mathrm{F1} = -0.0261$), confirming that both temporal components contribute independently to discriminative performance. The Global Mean imputation ablation produces the largest accuracy drop among non-MoE variants ($\Delta\mathrm{Acc} = -0.0259$, $\Delta\mathrm{AUC} = -0.0172$, $\Delta\mathrm{F1} = -0.0457$), consistent with global mean statistics overfitting to the population distribution on the smaller 1,377-patient cohort.
 
\begin{table}[ht]
  \centering
  \caption{%
    Component-wise ablation of \textsc{LongMoE} on OASIS-3
    (full modality).
    Results report Balanced Accuracy, AUC-ROC, Macro F1-score, and
    their deltas relative to the full model.
  }
  \label{tab:ablation_oasis}
  \renewcommand{\arraystretch}{1.2}
  \resizebox{0.7\textwidth}{!}{%
  \begin{tabular}{lcccccc}
    \toprule
    \textbf{Variant}
      & \textbf{Acc}
      & \textbf{AUC-ROC}
      & \textbf{F1-Score}
      & $\Delta$\textbf{Acc}
      & $\Delta$\textbf{AUC}
      & $\Delta$\textbf{F1} \\
    \midrule
    \textit{w/o} LE
      & 0.6271 & 0.8486 & 0.6276
      & $-0.0129$ & $-0.0066$ & $-0.0261$ \\
    \textit{w/o} VE
      & 0.6277 & 0.8296 & 0.6244
      & $-0.0123$ & $-0.0256$ & $-0.0293$ \\
    \textit{w/o} Imp-ZP
      & 0.6271 & 0.8486 & 0.6276
      & $-0.0129$ & $-0.0066$ & $-0.0261$ \\
    \textit{w/o} Imp-GM
      & 0.6141 & 0.8380 & 0.6080
      & $-0.0259$ & $-0.0172$ & $-0.0457$ \\
    \textit{w/o} Single Expert
      & 0.5991 & 0.8512 & 0.5954
      & $-0.0409$ & $-0.0040$ & $-0.0583$ \\
    \textit{w/o} Shared Experts
      & 0.5983 & 0.8450 & 0.6058
      & $-0.0417$ & $-0.0102$ & $-0.0479$ \\
    \textit{w/o} $\mathcal{L}_{\mathrm{bal}}$
      & 0.6317 & 0.8510 & 0.6466
      & $-0.0083$ & $-0.0042$ & $-0.0071$ \\
    \midrule
    \textsc{LongMoE}
      & \textbf{0.6400} & \textbf{0.8552} & \textbf{0.6537}
      & --- & --- & --- \\
    \bottomrule
  \end{tabular}%
  }
\end{table}

\paragraph{Ablation under Missing Modalities.}
Table~\ref{tab:ablation_oasis_masking} reports $\Delta\mathrm{AUC}$ as a function of the number of available modalities at inference time on OASIS-3. \textit{w/o} $\mathcal{L}_{\mathrm{bal}}$ and \textit{w/o} Shared Experts (SE) produce the largest drops at 1 modality ($-0.022$ and $-0.021$ respectively), confirming that balanced routing and the shared expert backbone are most critical when modality availability is minimal and the model must maintain discriminative performance from a single modality stream. \textit{w/o} Imp-GM is uniquely consistent across all missingness levels ($-0.016$, $-0.015$, $-0.017$), the only variant that degrades at full modality as well, reflecting the overfitting of population-level imputation statistics on this small cohort regardless of how many modalities are observed. By contrast, \textit{w/o} VE yields small positive $\Delta$AUC across all levels ($+0.001$, $+0.006$, $+0.004$), indicating that on OASIS-3's shorter visit sequences the temporal encoding contributes minimally to AUC under modality masking; and \textit{w/o} Single Expert shows a mixed profile ($-0.017$ at 1 modality, $+0.016$ at full), suggesting that sparse routing benefits information-rich full-modality inputs but converges slower than a monolithic expert under extreme missingness on this smaller dataset.

\begin{table}[ht]
  \centering
  \caption{%
    Ablation under modality masking on OASIS-3.
    Values report $\Delta\mathrm{AUC}$ relative to the full
    \textsc{LongMoE} model at each missingness level
    (1 modality through 3 modalities / full).
  }
  \label{tab:ablation_oasis_masking}
  \renewcommand{\arraystretch}{1.25}
  \resizebox{0.5\textwidth}{!}{%
  \begin{tabular}{lccc}
    \toprule
    \textbf{Variant}
      & \textbf{1 mod}
      & \textbf{2 mods}
      & \textbf{3 mods (full)} \\
    \midrule
    \textit{w/o} LE
      & $-0.009$ & $-0.005$ & $+0.003$ \\
    \textit{w/o} VE
      & $+0.001$ & $+0.006$ & $+0.004$ \\
    \textit{w/o} Imp-ZP
      & $-0.009$ & $-0.005$ & $+0.003$ \\
    \textit{w/o} Imp-GM
      & $-0.016$ & $-0.015$ & $-0.017$ \\
    \textit{w/o} Single Expert
      & $-0.017$ & $+0.000$ & $+0.016$ \\
    \textit{w/o} Shared Experts
      & $-0.021$ & $-0.011$ & $-0.010$ \\
    \textit{w/o} $\mathcal{L}_{\mathrm{bal}}$
      & $-0.022$ & $-0.015$ & $-0.004$ \\
    \bottomrule
  \end{tabular}%
  }
\end{table}

\subsection{Expert Specialisation Analysis}

\paragraph{OASIS-3.}
Figure~\ref{fig:expert_oasis} analyses routing behaviour on OASIS-3 ($K=8$: E1-E5 shared; E6-E8 individual, assigned to Imaging, Clinical, and Genetics respectively). The left plot shows a more heterogeneous stage distribution across shared experts than on ADNI, reflecting OASIS-3's more skewed CN/MCI/AD class split: E1 is the most CN-dominant (49\% CN), E2 and E3 are substantially MCI-enriched (30\% CN, rising MCI), and E4/E5 accumulate higher AD fractions (30\% AD), indicating that shared experts still capture a disease-severity gradient despite the more challenging class distribution. Individual experts show complementary stage profiles: E6 (Imaging) and E7 (Clinical) both receive approximately 40\% CN with rising MCI and AD fractions, while E8 (Genetics) is strongly AD-enriched (47\% AD), suggesting that genetics-only inputs are disproportionately associated with confirmed AD diagnoses in OASIS-3. The right plot shows that unlike ADNI, shared experts on OASIS-3 receive highly diverse modality combinations across I, C, G, and I,C combinations, consistent with the dataset's smaller cohort and three-modality structure; individual experts show clear modality-aligned activation: E6 (Imaging) concentrates on I and I,G; E7 (Clinical) on C and C,G; and E8 (Genetics) on G and I,C,G, confirming that context-conditioned routing recovers modality specialisation even in a smaller, more heterogeneous dataset.
 
\begin{figure}[ht]
  \centering
  \resizebox{0.7\textwidth}{!}{%
    \includegraphics[width=\textwidth]{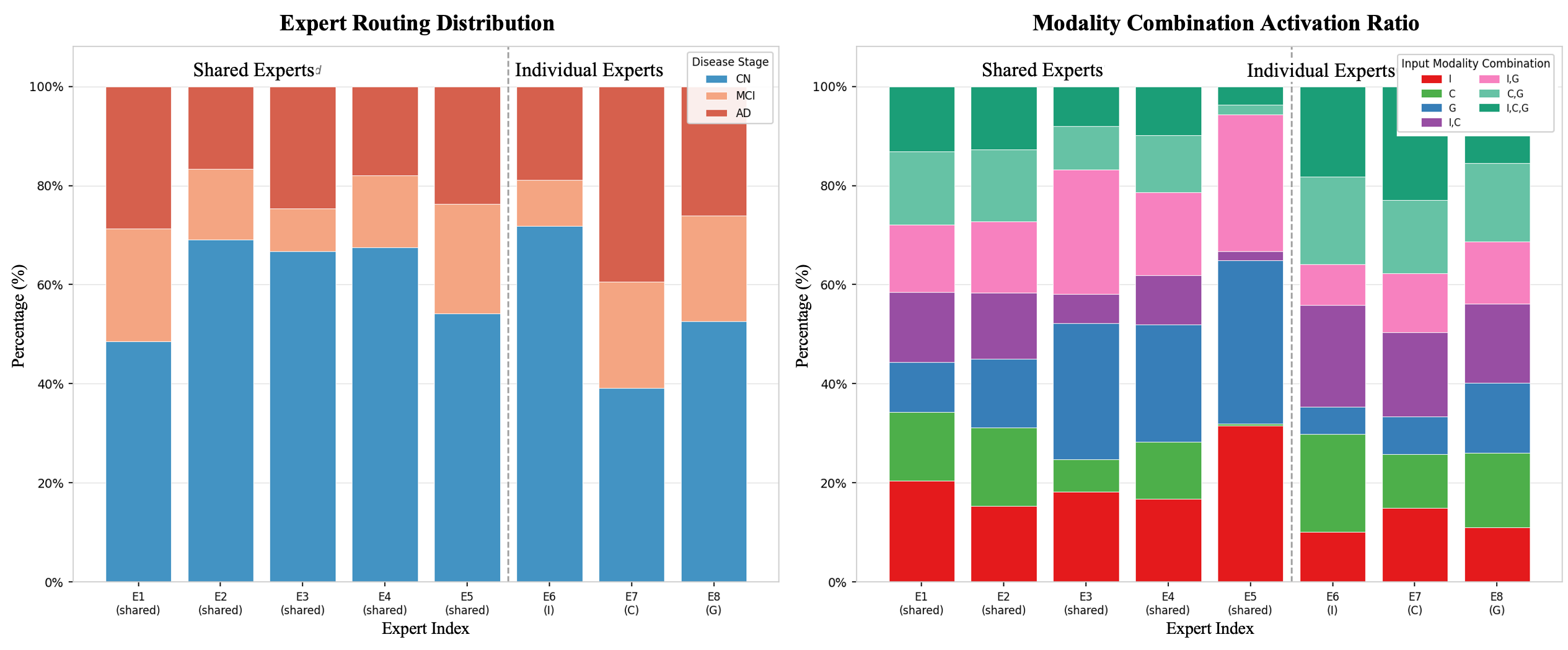}%
   }
  \caption{%
    Expert specialization analysis on OASIS-3.
    \textit{Left:} Expert routing distribution stratified by disease
    stage for shared and individual experts.
    \textit{Right:} Modality combination activation ratio per expert
    across the different OASIS-3 combinations.
  }
  \label{fig:expert_oasis}
\end{figure}
 
\paragraph{MIMIC-IV.}
Figure~\ref{fig:expert_mimic} analyses routing behaviour on MIMIC-IV ($K=8$: E1-E4 shared; E5-E8 individual, assigned to Clinical Notes, Lab Values, ICD Codes, and Demographics respectively). The left plot shows that all shared and individual experts are overwhelmingly dominated by the Survived class (88-90\%), consistent with the heavily imbalanced label distribution (80.6\% Survived), with shared experts showing a modest survival gradient: E1 (88\% Survived), E2 (91\%), E3 (89\%), and E4 (88\%), while individual experts exhibit slightly lower survival fractions (78-83\%), indicating that individual experts are relatively more active for the minority Died class than shared experts. The right plot shows that shared experts are predominantly activated by Notes and Notes-containing combinations, reflecting the dominance and universal availability of clinical notes in MIMIC-IV (100\% coverage), while individual experts exhibit clear modality-aligned specialisation: E5 (Notes) receives highly diverse combinations including N,C,D and N,L,C,D; E6 (Labs) concentrates on L and L-containing combinations; E7 (ICD Codes) shows broad activation across C, N,C, and L,C,D; and E8 (Demographics) is predominantly activated by D and Demographics-inclusive combinations. Together, these patterns confirm that \textsc{LongMoE}'s context-conditioned routing generalises expert specialisation meaningfully across both label types (multiclass staging and binary mortality) and modality semantics (neuroimaging vs.\ EHR), validating the architecture's subgroup-aware fusion design beyond the AD domain.
 
\begin{figure}[ht]
  \centering
  \resizebox{0.7\textwidth}{!}{%
    \includegraphics{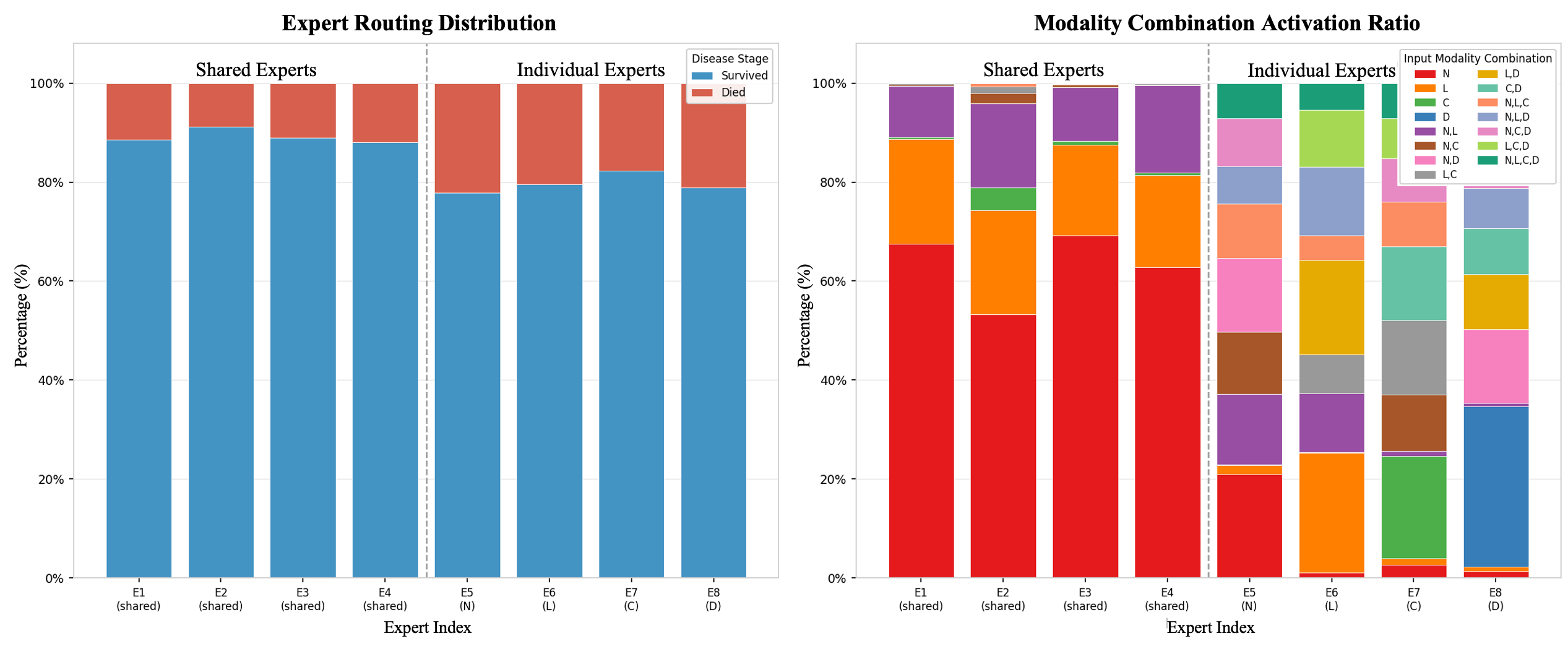}%
  }
  \caption{%
    Expert specialization analysis on MIMIC-IV.
    \textit{Left:} Expert routing distribution stratified by outcome
    for shared and individual experts.
    \textit{Right:} Modality combination activation ratio per expert
    across MIMIC-IV modality combinations.
  }
  \label{fig:expert_mimic}
\end{figure}

\subsection{Hyperparameter Configuration}
\label{app:hyperparams}
 
Tables~\ref{tab:hp_longmoe}--\ref{tab:hp_tf} report the final hyperparameter configurations for \textsc{LongMoE} and all the baselines after optimisation. All models share a common training infrastructure: gradient clipping at \texttt{max\_norm}$=1.0$, a patient-level stratified 70/15/15 data split, and five random seeds. \textsc{LongMoE} uses a higher epoch budget (100) with a cosine annealing scheduler to accommodate the longer convergence time of the trajectory-aware Transformer encoder; baseline models converge within 50 epochs. For \textsc{LongMoE}, the MoE layer comprises $K=8$ experts partitioned into 5 shared experts ($=\max(1, 8-3)$) and 3 modality- specific individual experts (I, C, G), with hard top-1 routing and attentional aggregation over visits. For Flex-MoE, the missing modality embedding bank has shape $(2^{n_{\mathrm{mods}}} - 1,\; n_{\mathrm{mods}},\; \texttt{hidden\_dim})$, covering all $2^L - 1$ observed modality patterns. For Lee-MMGRU, the GRU encoder (Stage 1) is trained end-to-end with Adam, and a Stage 2 L1-penalised logistic regression is fitted on the frozen GRU representations using the SAGA solver with $C \in \{0.01, 0.1, 1.0, 10.0\}$ cross-validated on the validation set and StandardScaler feature normalisation. All hyperparameters were selected by grid search over the validation set. Baseline configurations follow the original papers where available and are tuned otherwise.
 
\begin{table}[ht]
  \centering
  \caption{Hyperparameter configuration for \textsc{LongMoE}.}
  \label{tab:hp_longmoe}
  \renewcommand{\arraystretch}{1.2}
  \resizebox{0.4\textwidth}{!}{%
  \begin{tabular}{ll}
    \toprule
    \textbf{Hyperparameter} & \textbf{Value} \\
    \midrule
    Learning Rate           & $3 \times 10^{-4}$ \\
    Batch Size              & 32 \\
    Weight Decay            & $1 \times 10^{-4}$ \\
    Optimiser               & AdamW \\
    Epochs                  & 100 \\
    Scheduler               & CosineAnnealingLR \\
    Gradient Clipping       & \texttt{max\_norm} $= 1.0$ \\
    Token Dimension ($d_{\mathrm{tok}}$) & 128 \\
    Transformer Layers ($N$) & 2 \\
    Total Experts ($K$)     & 8 \\
    Shared Experts          & 5 ($= \max(1,\, 8-3)$) \\
    Modality Experts        & 3 (I,C,G) \\
    Top-$K$ Expert Selection & 1 \\
    Dropout                 & 0.2 \\
    Aggregation             & Attentional (\texttt{attn}) \\
    Output Classes          & 3 (CN / MCI / AD) \\
    \bottomrule
  \end{tabular}%
  }
\end{table}
 
\begin{table}[ht]
  \centering
  \caption{Hyperparameter configuration for Flex-MoE~\citep{yun2024flexmoe}.}
  \label{tab:hp_flexmoe}
  \renewcommand{\arraystretch}{1.2}
  \resizebox{0.5\textwidth}{!}{%
  \begin{tabular}{ll}
    \toprule
    \textbf{Hyperparameter} & \textbf{Value} \\
    \midrule
    Learning Rate           & $1 \times 10^{-4}$ \\
    Batch Size              & 8 \\
    Epochs                  & 50 \\
    Warm-up Epochs          & 5 \\
    Dropout                 & 0.5 \\
    Gate Loss Weight        & 0.01 \\
    Hidden Dim              & 128 \\
    Number of Experts       & 16 \\
    Number of Routers       & 1 \\
    Fusion Layers           & 1 \\
    Prediction Layers       & 1 \\
    Top-$K$                 & 4 \\
    Attention Heads         & 4 \\
    Missing Embed Shape     & $(2^{n_{\mathrm{mods}}}-1,\; n_{\mathrm{mods}},\; \texttt{hidden\_dim})$ \\
    Gradient Clipping       & \texttt{max\_norm} $= 1.0$ \\
    \bottomrule
  \end{tabular}%
  }
\end{table}
 
\begin{table}[ht]
  \centering
  \caption{Hyperparameter configuration for FuseMoE~\citep{han2024fusemoe}.}
  \label{tab:hp_fusemoe}
  \renewcommand{\arraystretch}{1.2}
  \resizebox{0.38\textwidth}{!}{%
  \begin{tabular}{ll}
    \toprule
    \textbf{Hyperparameter} & \textbf{Value} \\
    \midrule
    Learning Rate           & $1 \times 10^{-4}$ \\
    Batch Size              & 32 \\
    Weight Decay            & $1 \times 10^{-4}$ \\
    Optimiser               & AdamW \\
    Epochs                  & 50 \\
    Gradient Clipping       & \texttt{max\_norm} $= 1.0$ \\
    Hidden Dim              & 128 \\
    Number of Experts       & 8 \\
    Top-$K$                 & 2 \\
    Gating                  & Laplace distance-based \\
    Dropout                 & 0.1 \\
    \bottomrule
  \end{tabular}%
  }
\end{table}
 
\begin{table}[ht]
  \centering
  \caption{Hyperparameter configuration for MulT~\citep{tsai2019multimodal}.}
  \label{tab:hp_mult}
  \renewcommand{\arraystretch}{1.2}
  \resizebox{0.42\textwidth}{!}{%
  \begin{tabular}{ll}
    \toprule
    \textbf{Hyperparameter} & \textbf{Value} \\
    \midrule
    Learning Rate           & $1 \times 10^{-4}$ \\
    Batch Size              & 32 \\
    Weight Decay            & $1 \times 10^{-4}$ \\
    Optimiser               & AdamW \\
    Epochs                  & 50 \\
    Gradient Clipping       & \texttt{max\_norm} $= 1.0$ \\
    Model Dim ($d_{\mathrm{model}}$) & 128 \\
    Attention Heads         & 4 \\
    Transformer Layers      & 3 \\
    FFN Dim                 & 512 ($d_{\mathrm{model}} \times 4$) \\
    Activation              & GELU \\
    Dropout                 & 0.1 \\
    CLS Token               & Learnable \\
    Modality Position Embeddings & Learnable \\
    \bottomrule
  \end{tabular}%
  }
\end{table}
 
\begin{table}[ht]
  \centering
  \caption{Hyperparameter configuration for Lee-MMGRU~\citep{lee2019predicting}.
  Stage 1: GRU encoder. Stage 2: L1-penalised logistic regression
  on frozen representations.}
  \label{tab:hp_leegru}
  \renewcommand{\arraystretch}{1.2}
  \resizebox{0.37\textwidth}{!}{%
  \begin{tabular}{ll}
    \toprule
    \textbf{Hyperparameter} & \textbf{Value} \\
    \midrule
    \multicolumn{2}{l}{\textit{Stage 1 — GRU Encoder}} \\
    \midrule
    Learning Rate           & $1 \times 10^{-3}$ \\
    Batch Size              & 32 \\
    Weight Decay            & $1 \times 10^{-5}$ \\
    Optimiser               & Adam \\
    Epochs                  & 50 \\
    Gradient Clipping       & \texttt{max\_norm} $= 1.0$ \\
    GRU Hidden Dim          & 64 \\
    GRU Layers              & 1 \\
    Dropout                 & 0.1 \\
    \midrule
    \multicolumn{2}{l}{\textit{Stage 2 — Logistic Regression}} \\
    \midrule
    Penalty                 & L1 \\
    $C$ Candidates          & $\{0.01,\; 0.1,\; 1.0,\; 10.0\}$ \\
    Solver                  & SAGA \\
    Max Iterations          & 5{,}000 \\
    Feature Scaling         & StandardScaler \\
    \bottomrule
  \end{tabular}%
  }
\end{table}
 
\begin{table}[ht]
  \centering
  \caption{Hyperparameter configuration for TF (Tensor Fusion
  Network)~\citep{zadeh2017tensor}.}
  \label{tab:hp_tf}
  \renewcommand{\arraystretch}{1.2}
  \resizebox{0.58\textwidth}{!}{%
  \begin{tabular}{ll}
    \toprule
    \textbf{Hyperparameter} & \textbf{Value} \\
    \midrule
    Learning Rate           & $1 \times 10^{-4}$ \\
    Batch Size              & 32 \\
    Weight Decay            & $1 \times 10^{-4}$ \\
    Optimiser               & AdamW \\
    Epochs                  & 50 \\
    Gradient Clipping       & \texttt{max\_norm} $= 1.0$ \\
    Hidden Dim              & 128 \\
    Dropout                 & 0.1 \\
    Fusion                  & Outer product of $(\texttt{hidden\_dim}+1)$ vectors \\
    Fused Dim               & $\min\!\left((\texttt{hidden\_dim}+1)^{n_{\mathrm{mods}}},\;
                              \texttt{hidden\_dim} \times 4\right)$ \\
    \bottomrule
  \end{tabular}%
  }
\end{table}

\end{document}